\documentclass[twoside]{article}

\usepackage[accepted]{aistats2019}
\usepackage{algorithm}
\usepackage{algpseudocode}
\usepackage{mathtools}
\DeclarePairedDelimiter\abs{\lvert}{\rvert}%
\usepackage{amsthm}
\usepackage{cite}
 
\begin{document}



\twocolumn[

\aistatstitle{Unifying the stochastic and the adversarial Bandits with Knapsack }
\aistatsauthor{ }
\aistatsauthor{ Anshuka Rangi \And Massimo Franceschetti \And  Long Tran-Thanh }
\aistatsaddress{University of California, San Diego \And University of California, San Diego \And University of Southampton } ]
\newtheorem{theorem}{Theorem}
\newtheorem{lemma}[theorem]{Lemma}
\newtheorem{corollary}{Corollary}[theorem]
%\aistatsaddress{ Institution 1 \And  Institution 2 \And Institution 3 } ]
\begin{abstract}

This paper investigates the adversarial Bandits with Knapsack (BwK) online learning problem, where a player repeatedly chooses to perform an action, pays the corresponding cost, and receives a reward associated with the action. The player is constrained by the maximum budget $B$ that can be spent to perform  actions, and the rewards and the costs of the actions are assigned by an adversary. This problem has only been studied in the restricted setting where the reward of an action is greater than the cost of the action,  
while we  provide a solution in the general setting. Namely,  we propose EXP3.BwK, a novel algorithm that  achieves order optimal regret. 
We also propose EXP3++.BwK, 
which is order optimal in the adversarial BwK setup, and incurs an almost optimal expected regret with an additional factor of $\log(B)$ in the stochastic BwK setup.
Finally, we investigate the case of having large costs for the actions (i.e., they are comparable to the budget size $B$), and show that for the adversarial setting, achievable regret bounds can be significantly worse, compared to the case of having costs bounded  by a constant, which is a common assumption within the BwK literature.
\end{abstract}

\section{Introduction}

Multi-Armed Bandit (MAB) is a sequential decision making problem under uncertainty, that is based on balancing the trade-off  between  exploration and exploitation, i.e. ``the conflict between taking actions which yield immediate rewards and taking actions whose benefits will be seen later." A common feature in various applications of MAB is that the resources consumed during the decision making process are limited. For instance, scientists experimenting with alternative medical treatments may be limited by the number of patients participating in the study as well as by the cost of the material used in the treatments. Similarly, in web advertisements, a website experimenting with displaying advertisements is constrained by the number of users who visit the site as well as by the advertisers' budgets. A retailer engaging in price experimentation faces inventory limits along with a limited number of consumers. A model which incorporates a budget constraint on these supply limits is Bandits with Knapsack (BwK). 
This can be seen as a game between a player and an adversary (or environment)  that evolves  for $T$ rounds. The player is constrained by a budget $B$ on the resources consumed during the decision making process. The game terminates when the player runs out of budget, therefore  $T$ is dependent on $B$. At each round $t$, the player performs an action $i$ from a set of $K$ actions, pays a cost for the selected action $i$ from the budget $B$ and receives a reward in $[0,1]$ for the selected action $i$. The reward and the cost can vary from application to application. For example, in web advertisement, the reward is the click through rate and the cost is the space occupied by the advertisement on the web page. In medical trials, the reward is the success rate of the medicine and the cost corresponds to the cost of the material used.  

The Bandits with Knapsack problem can be classified into two categories: stochastic BwK and adversarial BwK. In stochastic BwK, the reward and the cost of each action is an i.i.d sequence over $T$ rounds drawn from a fixed  unknown distribution. In adversarial BwK, the sequence of the rewards and the costs associated with each action over $T$ rounds is assigned by an oblivious adversary  before the game starts.  The objective of the player is to minimize the expected regret, which is the  difference between the expectation of  the rewards received from the best fixed action in the hindsight and the sum of rewards received by the player's action selection strategy. 

%MAB has variety of applications like medical trials, recommendation systems, web searching and advertising, and communication networks. However, in these applications, the design of player's strategy is subjected to one or more resource constraints. For instance,  scientists experimenting with alternative medical treatments may be limited not only by the number of patients participating in the study but also by the cost of materials used in the treatments. A website experimenting with displaying advertisements is constrained not only by the number of users who visit the site but by the advertisers’ budgets. A retailer engaging in price experimentation faces inventory limits along with a limited number of consumers. In the literature, a model which incorporates a budget on these supply limits in the MAB problem is well known as Bandits with Knapsack (BwK). 
The stochastic BwK setting has been extensively studied in the literature \cite{tran2010epsilon,tran2012knapsack,ding2013multi,badanidiyuru2013bandits,agrawal2014bandits,tran2014efficient,agrawal2016linear, xia2016budgeted,sankararaman2017combinatorial,rangi2018multi}. The results in these works  can be broadly classified into two categories depending on the regret analysis. The problem dependent bound on the expected regret is $O(\log(B))$ % where $B$ is the budget of the player and $\Delta^*$ is the problem dependent constant (mathematically defined later)
 \cite{tran2012knapsack,ding2013multi,xia2016budgeted,zhou2017budget,rangi2018multi}, while the problem independent bound on the expected regret is $O(\sqrt{KB})$  \cite{agrawal2016linear,agrawal2014bandits,badanidiyuru2013bandits}.
 
Limited attention has been received by the adversarial BwK setting \cite{zhou2017budget}. In this setting, it has been assumed that the reward at round $t\leq T$ is greater than the cost at round $t\leq T$ for every action  over the duration of the game % i.e. $T$  
\cite{zhou2017budget}. 
Under this assumption, EXP3.M.B has been proposed and  proven  to be order optimal \cite{zhou2017budget}. We observe here that  the assumption on the reward being greater than the cost is uncommon in the literature of the BwK problem, and does not have any physical meaning in many applications. For example,  in web advertisement, the click through rate (i.e., reward) and the space occupied by the advertisement on the web page (i.e., cost) cannot be compared with each other. Likewise, in a medical trial, the reward is the success rate of the medicine and the cost corresponds to the cost of the material used, and the comparison of these values has no meaning. 
Thus, a key question  is how to design an  algorithm for the adversarial BwK in a general reward setting  that achieves order optimal regret guarantees.

Another key challenge is to provide a solution that is satisfactory for both stochastic and adversarial settings. In  many real-world situations, there is no information about whether the bandit model is used in a stochastic or adversarial manner. Thus, the deployed algorithm has to be able to perform well in both cases. Current algorithms in the adversarial BwK (e.g., EXP3.M.B), do not provide   optimal regret guarantees in the stochastic setting,  i.e.\ $O(\log(B))$, and algorithms in the stochastic BwK (e.g., KUBE), do not provide optimal regret guarantees  in the adversarial setting, i.e. $O(\sqrt{KB})$. Currently, there is no work   proposing a practical algorithm for both   settings. 
Finally, the literature of the BwK problem typically assumes that the costs are bounded   by a constant (i.e., they are independent of the budget $B$) and it is unknown  whether state-of-the-art regret bounds hold for the case of large costs (i.e., when costs  are comparable to the budget $B$).

In this framework, the contribution  of our work is three fold. First, we extend  EXP3, a classical algorithm, proposed for the adversarial MAB setup \cite{auer2002nonstochastic}, and propose EXP3.BwK, an algorithm for the adversarial BwK setup. We remove the assumption on the rewards and the costs previously used  in \cite{zhou2017budget} to obtain regret bounds and we  show that the expected regret of EXP3.BwK is $O(\sqrt{B K\log K })$. We also show the lower bound  $\Omega(\sqrt{K B})$ in the adversarial BwK setting. It follows that EXP3.BwK is order optimal.  Second, we  unify   the stochastic and the adversarial  settings by proposing EXP3++.BwK, a novel and practical  algorithm which works well in both of these settings. This algorithm incurs an expected regret of $O(\sqrt{BK\log K})$ and  $O(\log^2(B))$ in the adversarial and the stochastic BwK settings respectively. Note that the regret bound of EXP3++.BwK for the stochastic setting has an additional factor of $\log(B)$ in comparison to the optimal expected regret i.e. $O(\log(B))$. 
Thus, EXP3++.BwK exhibits an almost optimal behavior in both the stochastic  and the adversarial  settings.  
Table \ref{sample-table} summarizes these contributions and compares them with the other results in the literature. In the table, the problem-dependent parameter $\Delta(i)$ represents   the difference between the contributions of the optimal action and the action $i$, and is formally defined in the next section.
Finally, we show that if the maximum cost is bounded above by $B^{\alpha}$, where $\alpha \in [0,1]$, then the lower bound on the expected regret in the adversarial BwK setup scales at least linearly with the maximum cost, namely it is  $\Omega(B^{\alpha})$. 
This implies that when $\alpha > \frac{1}{2}$, it is impossible to achieve a regret bound of $O(\sqrt{B})$, which is order optimal in cases with small costs.  

\begin{table*}[t]
\begin{center}
\begin{tabular}{lll}
\textbf{Algorithm}  &\textbf{Upper bound} &\textbf{Lower bound}\\
\hline \\
%BalancedExploration \cite{badanidiyuru2013bandits}      &$O(\sqrt{Kg}+g\sqrt{K/B})$ &\\
%UCB for BwcR \cite{agrawal2014bandits}      &$O(\sqrt{\log(BK/\delta)}(\sqrt{Kg}+$&\\
                                             %&$g\sqrt{K/B}+K\sqrt{\log(BK/\delta)}))$ &\\
KUBE for BwK \cite{tran2012knapsack}            & $O(K\log(B)/\min_{i\in[K]}\Delta(i))$ &$\Omega(\log(B))$\\
B-KUBE  for Bounded BwK \cite{rangi2018multi}        &$O(K\log(B)/\min_{i\in[K]}\Delta(i))$  &$\Omega(\log(B))$\\
%&&\\
UCB-BV for variable cost \cite{ding2013multi}&$O(K\log(B)/\min_{i\in[K]}\Delta(i))$  &$\Omega(\log(B))$\\
UCB-MB for multiple plays\cite{zhou2017budget}& $O(K\log(B))$ &\\
EXP3.M.B\cite{zhou2017budget} &$O(\sqrt{K\log(K)B})$&$\Omega((1-1/K)^2 \sqrt{KB})$\\
EXP3.BwK (This work) &$O(\sqrt{K\log(K)B})$&$\Omega( \sqrt{KB})$\\
EXP3++.BwK in Adversarial setting (This work) & $O(\sqrt{K\log(K)B})$ &\\
EXP3++.BwK in Stochastic setting (This work)                     & $O(K\log^2(B)/\min_{i\in[K]}\Delta(i))$&\\
\end{tabular}
\end{center}
\caption{Contributions to the literature of BwK. } \label{sample-table}
\end{table*}
\subsection{Related Work}
In the MAB literature, the problem of finding one algorithm for both the stochastic and the adversarial setting has been referred as ``best of  both worlds" \cite{bubeck2012best,auer2016algorithm,seldin2014one,seldin2017improved,lykouris2018stochastic}. SAO, the first algorithm proposed in the literature of this problem, 
relies on the knowledge of the time horizon $T$, and  performs an irreversible switch to EXP3.P if the beginning of the game is estimated to exhibit an adversarial, or non-stochastic, behavior \cite{bubeck2012best}. The expected regret of SAO in the stochastic MAB setting is $O(\log^3(T))$, and in the adversarial MAB setting is ${ O}(\sqrt{T}\log^{2}(T))$. Using  ideas from SAO, a new algorithm SAPO was proposed \cite{auer2016algorithm}. SAPO exploited some novel criteria for the detection of the adversarial, or non-stochastic, behavior, and performs an irreversible switch to EXP3.P if such a behavior is detected. Thus, both SAO and SAPO  initially assume that the rewards are stochastic, and perform an irreversible switch to EXP3.P if this assumption is detected to be incorrect.
The expected regret of SAPO   is  $O(\log^2(T))$
in the stochastic MAB setting, and ${ O}(\sqrt{T\log(T^2)})$ in the adversarial MAB setting.  Later, EXP3++ was proposed  \cite{seldin2014one}. Unlike SAO and SAPO, this algorithm starts by assuming the rewards exhibit an adversarial, or non-stochastic, behavior and adapts itself as it encounters stochastic behavior on rewards. 
The analysis of EXP3++ was improved in \cite{seldin2017improved}, showing that the algorithm guarantees an expected regret of $O(\log^2(T))$ and $O(\sqrt{T})$ in the stochastic and the adversarial MAB settings respectively. %Thus, EXP3++ is the state of art algorithm for ``Best of Both Worlds" problem in MAB setup. 

The problem of stochastic bandits  corrupted with adversarial samples has been studied in the regime of small corruptions \cite{lykouris2018stochastic}. The algorithm proposed in this work utilizes the idea of active arm elimination based on  upper and lower confidence bound  of the estimated rewards. The work provides the regret analysis of the algorithm as the corruption $C$ is introduced in the rewards, and shows that the decay in performance  is order optimal in $C$. %This analysis is missing in the previous literature on ``Best of the both worlds" problem\cite{bubeck2012best,auer2016algorithm,seldin2014one,seldin2017improved}. However, the algorithm proposed in \cite{lykouris2018stochastic} is of importance only in the stochastic regime with small value of corruption $C$. This does not achieve the optimal regret guarantees as EXP3++ in the adversarial regime. 

The ``best of both worlds" problem has not been studied before in the BwK setting. 
%In this work, we propose an algorithm for  ``Best of both worlds" problem in the BwK setting which also requires studying the adversarial BwK setting.

\section{Problem Formulation}
A player can choose from a set of $K$ actions, and has a budget $B$. At round $t$, each action $i\in [K]$ is associated with a reward $r_{t}(i)\in [0,1]$ and a cost $c_{t}(i)\in [c_{min},c_{max}]$ with $c_{min} \leq c_{max}$. 
For now, we assume that $c_{max} = 1$, and will investigate the case of having larger costs in Section 5. At round $t$, the player performs an action $i_{t}\in [K]$, pays the cost $c_{t}(i_{t})$ and receives the reward $r_{t}(i_{t})$. The gain of a player's strategy $\mathcal{A}$ is defined as 
\[G(\mathcal{A})= \mathbf{E}\Big[\sum_{t=1}^{\tau(\mathcal{A})}r_{t}(i_{t})\Big],\]
where $\tau(\mathcal{A})$ is number of rounds after which the strategy $\mathcal{A}$ terminates.
The objective of a player is to design  $\mathcal{A}$ such that
\begin{equation}\label{eq:OptimizationProblem}
\begin{aligned}
\max_{\{i_{1},i_{2},\ldots,i_{\tau(\mathcal{A})}\}} G(\mathcal{A})\\ \mbox{s.t.  }
\mathbf{P}\Big(\sum_{t=1}^{\tau(\mathcal{A})}c_{t}(i_{t})\leq B\Big) =1.
\end{aligned}
\end{equation}
Note that $\tau(\mathcal{A})$ is dependent on the budget $B$. Let $\mathcal{A}^{*}$ be the algorithm that solves (\ref{eq:OptimizationProblem}). The expected regret of an algorithm $\mathcal{A}$ is defined as   
\begin{equation}\label{eq:regret}
R(\mathcal{A})=  G(\mathcal{A}^*)- G(\mathcal{A}).
\end{equation}
The optimization problem in (\ref{eq:OptimizationProblem}) is a knapsack problem, and is known to be NP-hard \cite{KelPfePis04}. Given that the rewards and the costs of all the actions are known and fixed for all $T$ rounds, the greedy algorithm $\mathcal{A}^{G}$ for solving (\ref{eq:OptimizationProblem}) makes an action selection in the decreasing order of the efficiency, defined as  $e(i)=r(i)/c(i)$ for an action $i\in[K]$, until the budget constraint in (\ref{eq:OptimizationProblem}) is satisfied. It can be shown that \cite{KelPfePis04}
\begin{equation}\label{eq:greedy}
    G(\mathcal{A}^{G})\leq G(\mathcal{A}^*)\leq G(\mathcal{A}^{G})+\max_{i\in [K]}e{(i)}. 
\end{equation}

In the stochastic setting, for all $t$ and $i\in [K]$, the reward $r_{t}(i)$ and the cost $c_{t}(i)$ of an action $i$ are  identically and independently distributed according to some  unknown distributions. The expected reward and the expected cost of an action $i$ are denoted by $\mu{(i)}$ and $\rho{(i)}$ respectively.
Thus, in the stochastic setting, the  efficiency of an action $i$ can be defined as $e(i)=\mu(i)/\rho(i)$. Using (\ref{eq:greedy}), the expected regret of an algorithm $\mathcal{A}$ simplifies to
\begin{equation}\label{eq:regretStochastic}
\begin{split}
    R(\mathcal{A})&\leq  \max_{i\in[K]}\frac{\mu{(i)}}{\rho{(i)}}\cdot (\tau(\mathcal{A}^{G})+1)- G(\mathcal{A})\\
    &=e(i^*)\cdot (\tau(\mathcal{A}^{G})+1)- G(\mathcal{A})\\
    &\leq \sum_{i\in [K]/\{i^*\}}\Delta(i)\mathbf{E}[N_{T}(i)],
\end{split}
\end{equation}
where $i^*=\mbox{argmax}_{i\in [K]}e(i)$, $\Delta (i)=e(i^*)-e(i)$, $N_{T}(i)$ is the number of times an action $i$ is selected in $T$ rounds, and $T=\max\{\tau(\mathcal{A}),\tau(\mathcal{A}^G)\}$. The definition in (\ref{eq:regretStochastic}) is consistent with the literature of stochastic BwK \cite{ding2013multi,tran2014efficient}. 

 In the adversarial setting, for all $t$, $r_{t}(i)$ and $c_{t}(i)$ are chosen by an  adversary before the game starts. 
In this setting, the efficiency of an action $i$ at round $t$ can be defined as $e_{t}(i)=r_{t}(i)/c_{t}(i)$. Therefore, the expected regret simplifies to
\begin{equation}\label{eq:regretAdv}
\begin{split}
        R(\mathcal{A})&\leq  \frac{B}{T(i^*)}\sum_{t=1}^{T(i^*)} \frac{r_{t}(i)}{c_{t}(i)}- G(\mathcal{A}),\\
    &\leq \mathbf{E}\Bigg[z(\mathcal{A})\bigg(\sum_{t=1}^{T(i^*)} e_{t}(i^*)-\sum_{t=1}^{\tau(A)}e_{t}(i_{t})\bigg) \Bigg],
\end{split}
\end{equation}
where $T{(i)}$ is the number of rounds for which the game is feasible in the budget $B$ when a fixed action $i\in [K]$ is performed, $i^*=\mbox{argmax}_{i\in [K]}\sum_{t=1}^{T(i)}e_{t}(i)$ is the optimal action in the hindsight, 
\[z(\mathcal{A})=\max\Bigg\{\frac{B}{T(i^*)},\frac{B(\mathcal{A})}{\tau(\mathcal{A})}\Bigg\}\]
is the maximum cost per round, $B(\mathcal{A})$ is the budget utilized by the algorithm $\mathcal{A}$,
and the inequality follows from (\ref{eq:greedy}). 
The expected regret is bounded by the expectation of the efficiency regret scaled by the maximum of the cost spent per round by the optimal action $i^*$, and  the cost spent per round by the algorithm $\mathcal{A}$, where the efficiency regret is  the sum of the  rewards per unit cost associated to the optimal action minus the sum of the rewards per unit cost associated to the actions performed by the algorithm $\mathcal{A}$. 

%In this work, we first design an algorithm that is order optimal in adversarial BwK setting. Later, we extend the ideas developed in the adversarial BwK  to design an algorithm  which achieves 
%almost optimal performance guarantees in both the stochastic and the adversarial regime of BwK. 

\section{Adversarial BwK}
In this section, we propose the algorithm EXP3.BwK for the adversarial BwK setting, and show that it is order optimal. 

\begin{algorithm}[t]
\begin{algorithmic} 
%\KwResult{Write here the result }
\State Initialization: $\gamma$ ; For all $i\in [K]$, $w_{1}(i)=1$,  and ${\hat{e}}_{1}(i)=0$; $t=1$; 
\While{$B>0$}
\State $W_{t}={\sum_{j\in [K]}w_{t}(j)}$
\State Update $p_{t}(i)={(1-\gamma)w_{t}(i)}/W_{t} + \gamma/K$
\State Choose $i_{t} = i $ with probability  $p_{t}(i)$.
\State Observe $(r_{t}(i_{t}),c_{t}(i_{t)})$
\If{$c_{t}(i_{t})>B$}
\State  exit;
\EndIf
\State $B=B-c_{t}(i_{t})$

\State For all $i\in[K]$, ${\hat{e}}_{t}(i)=r_{t}(i)\textbf{1}(i=i_{t})/p_{t}(i)c_{t}(i)$.

\State $w_{t+1}(i)=w_{t}(i)\cdot \exp(\gamma c_{min}\cdot\hat{e}_{t}(i)/K )$
\State t=t+1
 \EndWhile
 \caption{EXP3.BwK}
  \label{alg:Exp3.bwk}
 \end{algorithmic}
\end{algorithm}
Similar to EXP3, EXP3.BwK maintains a set of time-varying weights $w_{t}(i)$ for each action $i\in [K]$. At each round $t$, an action $i_{t}=i$ is selected with probability $p_{t}(i)$ which is dependent on two parameters: the time-varying weights $w_{t}(i)$ and an exploration constant $\gamma/K$. Following the selection of the action $i_{t}$, the algorithm pays the cost $c_{t}(i_{t})$. If  the cost $c_{t}(i_{t})$ is greater than the remaining budget of the algorithm, then the algorithm terminates without attempting to find other feasible actions which can be performed using the remaining budget. In  EXP3.BwK, the efficiency $e_{t}(i)=r_{t}(i)/c_{t}(i)$ is used as a measure of the contribution from an action $i \in [K]$ at round $t$. The empirical estimate of the efficiency $\hat{e}_{t}(i)$ (defined in Algorithm \ref{alg:Exp3.bwk}) is used to update the weight $w_{t}(i)$ of the action $i$. For all $i\in [K]$,  the difference in the weights $w_{t}(i)$ and $w_{t-1}(i)$ is controlled by scaling $\hat{e}_{t}(i)$ with $\gamma c_{min}$, which ensures that the $\gamma c_{min}\hat{e}_{t}(i)\leq 1$. The probability $p_{t}(i)$ is dependent on $w_{t}(i)$ and the exploration constant $\gamma/K$. In the probability $p_{t}(i)$, the weight $w_{t}(i)$ is responsible for the exploitation as it favors the selection of an action with higher cumulative efficiencies i.e. $\sum_{n=1}^{t}\hat{e}_{t-1}(i)$ observed until round $t-1$.  On contrary, the exploration constant $\gamma/K$ ensures that the player is always exploring with a positive probability in search of the optimal action $i^*$. This balances the trade-off between exploration and exploitation.

In the literature of the adversarial BwK setup \cite{zhou2017budget}, it has been assumed that for all actions $i\in [K]$ and for all $t$, $r_{t}(i)\geq c_{t}(i)$. This allows the use of a different efficiency measure $r_{t}(i)-c_{t}(i)$, which is linear in both the reward and the cost of an action $i$, thus simplifying the proofs \cite{zhou2017budget}. In many real life applications, the rewards and the costs are on different scales, and cannot be compared by an inequality operator. For example, in a recommendation system, a recommender is constrained by the total space available on the web page which corresponds to the budget $B$, the space occupied by each item corresponds to  its cost, and  the click rate of each item corresponds to  its reward. In this case, the space (cost)  of the item and the click rate (reward) of the item are not comparable. Likewise, the efficiency measure  $r_{t}(i)-c_{t}(i)$ which compares the reward and the cost of an action $i$ on a linear scale,  is questionable and provides no intuition about the optimality of an action. 
In EXP3.BwK, we use a different efficiency measure $r_{t}(i)/c_{t}(i)$ for tracking the contributions of each action $i\in [K]$. The use of this measure is motivated from the greedy algorithm $\mathcal{A}^{G}$, and its performance guarantees with respect to the optimal solution (see (\ref{eq:greedy}) and (\ref{eq:regretAdv})). The advantages of using this measure are two folds. First, it eliminates the need of the assumption in \cite{zhou2017budget}. Second, it can track  $G(\mathcal{A})$ of the algorithm $\mathcal{A}$  irrespective of the measure of the rewards and the costs.  

The following theorem provides the performance guarantees of EXP3.BwK in terms of the expected regret, and shows that it is sublinear in the budget $B$.
\begin{theorem}\label{thm:EXP3.BwK}
For $\gamma=\sqrt{c_{min} K\log(K)/B(e-1)}$, the expected regret, as defined in (\ref{eq:regretAdv}), of the algorithm EXP3.BwK is at most
\begin{equation}
\begin{split}
        R({E})&\leq  2\sqrt{\Bigg((e-1)+(e-2)\frac{K}{B}\Bigg)\frac{BK\log(K)}{c_{min}^3}},\\
      %  &=  2\sqrt{\Bigg((e-1)+(e-2)\frac{K}{B}\Bigg)\frac{gK\log(K)}{c_{min}}},\
\end{split}
\end{equation}
where $ {E}$ denotes EXP3.BwK.% and the last equality follows from the fact that the maximum gain of a knapsack problem is $g=B/c^2_{min}$.
\end{theorem}
\begin{proof}
We briefly discuss the key ideas of the proof here, and its detailed version is presented in the supplementary material. The expected regret is at most
\begin{equation}\label{eq:simplify}
\begin{split}
     &\frac{B}{T(i^*)}\sum_{t=1}^{T(i^*)} \frac{r_{t}(i)}{c_{t}(i)}- G(E),\\
     &\stackrel{}{\leq}\mathbf{E}\Bigg[z(E)\Bigg(\sum_{t=1}^{T(i^*)} \frac{r_{t}(i^*)}{c_{t}(i^*)}- \sum_{t=1}^{\tau(E)} \frac{r_{t}(i_{t})}{c_{t}(i_{t})}\Bigg)\Bigg],\\
\end{split}
\end{equation}
where 
%\[i^{*}=\mbox{argmax}_{i\in [K]} \sum_{t=1}^{T(i)}\frac{r_{t}(i)}{c_{t}(i)}\]
%is the optimal fixed action in hindsight,
$\tau(E)$ is the stopping time of EXP3.BwK and $B(E)$ is the budget utilized by EXP3.BwK, and 
\[z(E)=\max\Bigg\{\frac{B}{T(i^*)},\frac{B(E)}{\tau(E)}\Bigg\}.\]
% In (\ref{eq:simplify}), $(a)$ follows from the fact that 
% \begin{equation}\label{eq:sumDiff}
%     \Bigg(\sum_{t=1}^{T(i^*)} \frac{r_{t}(i^*)}{c_{t}(i^*)}- \sum_{t=1}^{\tau(E)} \frac{r_{t}(i_{t})}{c_{t}(i_{t})}\Bigg)
% \end{equation}
% is the difference between the cumulative rewards over all the rounds per unit cost of the optimal action $i^{*}$ and EXP3.BwK, and $z(E)$ is the maximum cost per round.
 
Using (\ref{eq:simplify}), the expected regret can be bounded by showing that  
\begin{equation}\label{eq:sum}
\begin{split}
    &\mathbf{E}\Bigg[\Bigg(\sum_{t=1}^{T(i^*)} \frac{r_{t}(i^*)}{c_{t}(i^*)}- \sum_{t=1}^{\tau(E)} \frac{r_{t}(i_{t})}{c_{t}(i_{t})}\Bigg)\Bigg]\\
    &\leq  2\sqrt{\frac{((e-1)B+(e-2)K)K\log(K)}{c_{min}^3}},
\end{split}
\end{equation}
 and $z(E)\leq 1$. 
\end{proof}
The key challenge in the proof of Theorem \ref{thm:EXP3.BwK} is that  the two summations in (\ref{eq:sum}) corresponding to the optimal action $i^*$ and the algorithm EXP3.BwK are along the different  time scales,   $T(i^*)$ and $\tau(E)$ respectively. This requires the analysis to be split into two  cases: $T(i^*)\geq \tau(E)$ and $T(i^*)\leq \tau(E)$. The analysis for these cases is based on the inference that $B(E)>B-K$ because the algorithm EXP3.BwK terminates at round $t$ if and only if the remaining budget is insufficient  to pay the cost $c_{t}(i_{t})\leq 1$. Hence, we can bound the difference between the two time scales i.e. $T(i^*)$ and $ \tau(E)$ as follows:
\begin{equation} \label{eq:StopRule}
    \abs{T(i^*)- \tau(E)}\leq \frac{K}{c_{min}}. 
\end{equation}
It follows that the difference between the number of rounds of the optimal action $i^*$ and EXP3.BwK is bounded by a fixed constant independent of the budget $B$. Hence, the regret of the algorithm  due to this difference in (\ref{eq:StopRule})   is at most $K/c_{min}^2$, and does not introduce any dependency on the budget $B$.

 The following theorem provides the lower bound on the expected regret in the adversarial BwK setting.
 \begin{theorem}\label{thm:LowerBound}
 For any player's strategy $\mathcal{A}$, there exists an adversary for which the expected regret of the algorithm $\mathcal{A}$ is at least $\Omega(\sqrt{KB/c^2_{min}})$.
 \end{theorem}
 \begin{proof}
The adversary chooses the optimal action $i^*$ uniformly at random from the set of  $K$ actions. For the action $i^{*}$ and for all $t$, the reward  $r_{t}(i^{*})$ is assigned using an independent Bernoulli random variable with expectation $0.5+\epsilon$, where $\epsilon=\sqrt{Kc_{min}/B}$. For all $i\in [K]/\{i^*\}$ and for all $t$,  the reward $r_{t}(i)$ is assigned using an independent Bernoulli random variable with expectation $0.5$. For all $i\in [K]$ and for all  $t$, the adversary assigns cost $c_{t}(i)=c_{min}$. The remaining proof is along the same lines as the lower bound on the expected regret in the MAB setup \cite{auer2002nonstochastic}. 
 \end{proof}

 By comparing the results in Theorem \ref{thm:EXP3.BwK} and Theorem \ref{thm:LowerBound}, the expected regret of EXP3.BwK has an additional factor of $1/\sqrt{c_{min}}$, and is order optimal in the budget $B$. This also highlights an important feature of an alternate class of algorithms in the BwK setup. Consider a new class of algorithms  $\mathcal{G}$ which looks for an alternative action to perform after the algorithm is unable to pay the cost $c_{t}(i_t)$ at round $t$  in order to utilize the remaining budget effectively. Since EXP3.BwK  terminates if it is unable to pay the cost $c_{t}(i_t)$,  EXP3.BwK does not belong to $\mathcal{G}$, and is still order optimal in the budget $B$. Therefore, the expected regret of this new class of algorithms $\mathcal{G}$ will have same dependency as that of EXP3.BwK on the budget $B$.  Additionally, the difference between the expected regret of EXP3.BwK and the class of algorithms $\mathcal{G}$
 will be at most a constant i.e. $K/c_{min}^2$, independent of $B$ (see (\ref{eq:StopRule})). The class of algorithms $\mathcal{G}$ faces the additional challenge of designing an appropriate criterion for the termination of the algorithm because the costs are assigned by the adversary. 
 
 The ideas developed in EXP3.BwK, particularly the measure of the efficiency $r_{t}(i)/c_{t}(i)$ forms form the basis of designing an algorithm which achieves almost optimal performance guarantees in both the stochastic and the adversarial BwK settings.
\section{One practical algorithm for both stochastic and adversarial BwK}
\begin{algorithm}[t]
\begin{algorithmic}
%\KwResult{Write here the result }
\State Initialization: For all $i\in [K]$, $w_{1}(i)=1$, ${\hat{e}}_{1}(i)=0$, ${\Bar{e}}_{1}(i)=0$, ${N}_{1}(i)=1$ $\delta_{1}(i)>0$; $t=1$, $\gamma_{t}=0.5 c_{min}\sqrt{\log(K)/tK}$; 
\State Perform each action once and update for all $i\in[K]$, $\Bar{e}_{1}(i)=r_{1}(i)/c_{1}(i)$, $B=B-\sum_{i\in[K]}c_{1}(i)$ and $t=K+1$.
\While{$B>0$}
\State For all $i\in[K]$, update:
\State \qquad $\mbox{UCB}_{t}(i) $ (see (\ref{eq:ucb}))
\State \qquad $\mbox{LCB}_{t}(i) $ (see (\ref{eq:lcb}))
\State \qquad $\hat{\Delta}_{t}(i)$ (see (\ref{eq:gapEstimate}))
\State \qquad $\delta_{t}(i)=\beta \log(t)/(t\hat{\Delta}_{t}(i)^2) $
\State  \qquad$\epsilon_{t}(i)=\min\{1/2K,0.5\sqrt{\log(K)/t},\delta_{t}(i)\}$
%\State \qquad ${p}_{t}(i)={(1-\gamma)w_{t}(i)}/W_{t} + \gamma/K$
\State \qquad ${p}_{t}(i)=\frac{\exp(-\gamma_{t}\hat{L}_{t-1}(i))}{\sum_{j\in[K]}\exp(-\gamma_{t}\hat{L}_{t-1}(j) )}$
\State \qquad $\Tilde{p}_{t}(i)={(1-\sum_{j\neq i}\epsilon_{t}(j)){p}_{t}(i)} + \epsilon_{t}(i)$
\State Choose $i_{t} = i $ with probability  $\Tilde{p}_{t}(i)$.
\State Observe $(r_{t}(i_{t}),c_{t}(i_{t)})$
\If{$c_{t}(i_{t})>B$}
\State  exit;
\EndIf
\State $B=B-c_{t}(i_{t})$

\State For all $i\in[K]$, update:
\State \qquad${\hat{e}}_{t}(i)=r_{t}(i)\textbf{1}(i=i_{t})/\Tilde{p}_{t}(i)c_{t}(i)$.
\State \qquad${\hat{\ell}}_{t}(i)=\textbf{1}(i=i_{t})/c_{min}\Tilde{p}_{t}(i)-\hat{e}_{t}(i)$.
\State \qquad $\hat{L}_{t}(i)=\sum_{n=1}^{t}\hat{\ell}_{n}(i)$
\State \qquad$N_{t}(i)=N_{t-1}(i)+\textbf{1}(i=i_{t})$.
\State \qquad$\Bar{r}_{t}(i)=\sum_{n=1}^{t}r_{n}(i)\textbf{1}(i=i_{n})/N_{t}(i)$  
\State \qquad$\Bar{c}_{t}(i)=\sum_{n=1}^{t}c_{n}(i)\textbf{1}(i=i_{n})/N_{t}(i)$
\State \qquad$\Bar{e}_{t}(i)=\Bar{r}_{t}(i)/\Bar{c}_{t}(i)$ 

\State t=t+1
 \EndWhile
 \caption{EXP3++.BwK}
 \label{alg:ThresholdEXP3}
 \end{algorithmic}
\end{algorithm}
In this section, we propose the algorithm EXP3++.BwK (Algorithm \ref{alg:ThresholdEXP3}), and show that it achieves almost optimal performance guarantees in both the stochastic and the adversarial BwK settings. 

Before discussing the algorithm EXP3++.BwK, let us briefly focus on the fundamental difference between the optimal algorithms in the stochastic and the adversarial BwK settings. In the stochastic BwK setting, the algorithms focus on exploration in the initial stage until a reliable estimate of the expected rewards and expected costs is achieved. Then, the algorithms focus on  exploitation, and perform exploration with a small  probability. For example, in UCB type of algorithms, the probability of exploration decays as $1/t^{2}$ with round $t$ \cite{tran2012knapsack,ding2013multi,rangi2018multi}. In greedy algorithms, the probability of exploration is zero after a fixed round (or time instance) \cite{tran2010epsilon,tran2014efficient}. On the contrary, in the adversarial regime, the algorithms are always exploring, and looking for the actions with higher contributions \cite{auer2002nonstochastic}. For instance, in EXP3.BwK,  the exploration constant $\gamma/K$ does not change with the round $t$, and it is dependent on the total number of rounds i.e. $\Theta(B)$ in the BwK setup. 

For all action $i\in [K]$, EXP3++.BwK maintains an Upper Confidence Bound (UCB) $\mbox{UCB}_{t}(i)$ and a Lower Confidence Bound (LCB) $\mbox{LCB}_{t}(i)$  on the efficiency $e(i)$, where 
\begin{equation}\label{eq:ucb}
   \mbox{UCB}_{t}(i) = \min\Bigg\{\frac{1}{c_{min}},\Bar{e}_{t}(i)+\frac{(1+1/\lambda) \eta_{t}(i)}{\lambda- \eta_{t}(i)}\Bigg\},
\end{equation}
\begin{equation}\label{eq:lcb}
   \mbox{LCB}_{t}(i) = \max\Bigg\{0,\Bar{e}_{t}(i)-\frac{(1+1/\lambda) \eta_{t}(i)}{\lambda- \eta_{t}(i)}\Bigg\},
\end{equation}
\begin{equation}
    \eta_{t}(i) =\sqrt{\frac{\alpha \log(K^{1/\alpha}t)}{2N_{t}(i)}},
\end{equation}
$\lambda\leq c_{min}$ and $N_{t}(i)$ is the number of times an action $i$ has been chosen until round $t$. The UCB and the LCB on an action $i$ are used to estimate $\Delta(i)$. The estimate of this gap  at round $t$ is defined as
\begin{equation}\label{eq:gapEstimate}
     \hat{\Delta}_{t}(i)=\max\{0,\max_{j\neq i}\mbox{LCB}_{t}(j)-\mbox{UCB}_{t}(i)\}.
\end{equation}
It can been shown that for all $i\in [K]$, in the stochastic BwK setting, we have 
\[\frac{\Delta(i)}{2}\leq \hat{\Delta}_{t}(i)\leq \Delta(i),\]
with high probability as $t\to \infty$. Thus,  $\hat{\Delta}_{t}(i)$ is a reliable estimate of $\Delta(i)$. For all $i\in [K]$, the estimate of the gap $\hat{\Delta}_{t}(i)$ is used to design the exploration parameter $\epsilon_{t}(i)$ in the sampling probability  $\tilde{p}_{t}(i)$ where $\tilde{p}_{t}(i)$ is the probability of choosing an action $i$ at round $t$. In the stochastic BwK setup, since $\Delta(i^*)=0$, the exploration parameter $\epsilon_{t}(i^*)$  of the optimal action $i^*$ tends to zero, and favors its selection. Unlike EXP3.BwK, the exploration parameter $\epsilon_{t}(i)$ varies with  $t$. Additionally, the sampling probability $\tilde{p}_{t}(i)$ is dependent on both the estimates of the efficiencies $\hat{e}_{t}(i)$ and $\Bar{e}_{t}(i)$ where $\hat{e}_{t}(i)$ and $\Bar{e}_{t}(i)$   are crucial in the adversarial BwK setting (see EXP3.BwK) and  the stochastic BwK setting respectively.  In the sampling probability $\tilde{p}_{t}(i)$, $\hat{e}_{t}(i)$ controls the exploitation performed by the algorithm through $p_{t}(i)$, and $\bar{e}_{t}(i)$ controls the exploration performed by the algorithm through the exploration parameter $\epsilon_{t}(i)$.

The following theorem provides the performance guarantees of EXP3++.BwK in the stochastic BwK setting. 
\begin{theorem}\label{thm:stochasticEXP3}
In the stochastic BwK setting, for $\alpha=3$ and $\beta=256/c_{min}^2$, the expected regret of the EXP3++.BwK is at most
\[R(F)= O\Bigg(\sum_{i:\Delta(i)>0}\frac{\log^2(B/c_{min})}{c_{min}^2\Delta(i)}\Bigg),\]
where $F$ denotes the algorithm EXP3++.BwK. 
\end{theorem}
\begin{proof}
The expected regret of the algorithm can be bounded by
\[R(F)\leq \sum_{i\in [K]/\{i^*\}}\Delta(i)\mathbf{E}[N_{T}(i)],\]
 where $T \leq B/c_{min}$ is the number of rounds at the termination of the algorithm. 
We can then bound the expected number of times $\mathbf{E}[N_{T}(i)]$ an action $i\neq i^*$ is selected by the algorithm. Since the probability of the selection of an action $i$ is $\tilde{p}_{t}(i)$, we have
 \begin{equation}\label{eq:numTime}
     \mathbf{E}[N_{T}(i)]\leq \mathbf{E}[\sum_{t=1}^{T}\epsilon_{t}(i)+p_{t}(i)].
 \end{equation}
We now bound the two terms in the right hand side of (\ref{eq:numTime}) in the stochastic BwK setting. 
First, we show that the estimate $\hat{\Delta}_{t}(i)$ is a reliable estimate of $\Delta(i)$, i.e. 
\begin{equation}
      \mathbf{P}(\hat{\Delta}_{t}(i)\geq \Delta(i))\leq \frac{1}{t^{\alpha-1}},
\end{equation}
\begin{equation}\label{eq:11}
\begin{split}
    \mathbf{P}\Bigg(\hat{\Delta}_{t}(i)&\leq \frac{\Delta(i)}{2}\Bigg)\leq \Bigg(\frac{\log t}{tc_{min}^2\Delta(i)^2}\Bigg)^{\alpha - 2}+2\Bigg(\frac{1}{t}\Bigg)^{\frac{\beta c_{min}^2}{8}}\\
    &\qquad\qquad\qquad+\frac{2}{Kt^{\alpha -1}}.
\end{split}
\end{equation}
These results can be used to prove that
\begin{equation}\label{eq:12}
    \mathbf{P}\Bigg(
    \tilde{\Delta}_{t}(i)\leq \frac{t\Delta(i)}{2}\Bigg)\leq \Bigg(\frac{\log(t)}{tc_{min}^2\Delta(i)^2}\Bigg)^{\alpha-2} +\frac{1}{t},
\end{equation}
where $\tilde{\Delta}_{t}(i)=\sum_{n=1}^{t} (\hat{\ell}_{n}(i)-\hat{\ell}_{n}(i^*))$. Since 
\[p_{t}(i)\leq \exp(-\gamma_{t}\tilde{\Delta}_{t}(i)),\]
(\ref{eq:12}) is used to bound $\sum_{t=1}^{T}\mathbf{E}[p_{t}(i)]$, and we have
\[\sum_{t=1}^{T}\mathbf{E}[p_{t}(i)]=O\Bigg(\frac{\log^2(B/c_{min})}{c_{min}^2\Delta(i)^2}\Bigg).\]
 Using the definition of $\epsilon_{t}(i)$ and (\ref{eq:11}), we have  
 \[\sum_{t=1}^{T}\mathbf{E}[\epsilon_{t}(i)]=O\Bigg(\frac{\log^2(B/c_{min})}{c_{min}^2\Delta(i)^2}\Bigg).\]
 Hence, the statement of the theorem follows. The detailed version of the proof is in supplementary material. 
\end{proof}
 In Theorem \ref{thm:stochasticEXP3}, EXP3++.BwK incurs an expected regret of $O(\log^2(B/c_{min}))$, whereas the optimal regret guarantees in the stochastic BwK setting are $O(\log(B/c_{min}))$\cite{tran2012knapsack,ding2013multi,rangi2018multi}. Thus, EXP3++.BwK has an additional factor of $\log(B/c_{min})$ in comparison to the results in the literature. This additional factor is also common in the literature of MAB  \cite{seldin2014one,lykouris2018stochastic}. The following theorem provides the performance guarantees of EXP3++.BwK in the adversarial BwK setting.
 
\begin{theorem}\label{thm:adversarialEXP3}
In the adversarial BwK setting, the expected regret of the EXP3++.BwK is at most
\begin{equation}
   R(F)\leq \sqrt{\frac{6BK\log(K)}{c_{min}^3}}.
\end{equation}
\end{theorem}
\begin{proof}
Similar to the proof of Theorem \ref{thm:EXP3.BwK}, we bound
\begin{equation}
\begin{split}
    \mathbf{E}\Bigg[z(E)\Bigg(\sum_{t=1}^{T(i^*)} \frac{r_{t}(i^*)}{c_{t}(i^*)}- \sum_{t=1}^{\tau(E)} \frac{r_{t}(i_{t})}{c_{t}(i_{t})}\Bigg)\Bigg].\\
\end{split}
\end{equation}
We show that 
\begin{equation}
\begin{split}
    &\mathbf{E}\Bigg[\Bigg(\sum_{t=1}^{T(i^*)} \frac{r_{t}(i^*)}{c_{t}(i^*)}- \sum_{t=1}^{\tau(E)} \frac{r_{t}(i_{t})}{c_{t}(i_{t})}\Bigg)\Bigg]\\
    &\leq  \sqrt{\frac{6BK\log(K)}{c_{min}^3}},
\end{split}
\end{equation}
 and $z(E)\leq 1$. The detailed version of the proof is in supplementary material. 
\end{proof}
Thus, like EXP3.BwK, EXP3++.BwK is order optimal in the adversarial BwK setting. The challenges in the proof of Theorem \ref{thm:stochasticEXP3} and Theorem \ref{thm:adversarialEXP3} are addressed  in a similar way as that of Theorem \ref{thm:EXP3.BwK}.
In conclusion, using Theorem \ref{thm:stochasticEXP3} and Theorem \ref{thm:adversarialEXP3}, the EXP3++.BwK is order optimal in the adversarial BwK setting and has an additional factor of $\log(B/c_{min})$ in the stochastic BwK setting. 

\section{BwK with unbounded cost}
Assuming the cost is bounded by unity (i.e., $c_{max} = 1$),  Theorem \ref{thm:LowerBound} provides the dependence of the expected regret on the minimum cost $c_{min}$ in the adversarial BwK setup. In this section, we  discuss the scaling of the lower bound on the expected regret with respect to the maximum cost $c_{max}$ in the adversarial BwK setup.
\begin{theorem} \label{thm:Cost} Suppose that $c_{\max} = B^{\alpha}$. For any algorithm $\mathcal{A}$, there exists an adversary such that the expected regret of the algorithm is at least $\Omega(B^{\alpha})$.
\end{theorem}
\begin{proof}
Let the number of actions be $K=2$, and the actions be $i_{1},i_{2}$. 
The adversary chooses the optimal action $i^*$ uniformly at random from these two actions. Let $t^*= B - B^{\alpha}$.
For all $t\leq t^*$ rounds, the adversary assigns $r_{t}(i_{1})=r_{t}(i_{2})=0$ and $c_{t}(i_{1})=c_{t}(i_{2})=1$ to both the actions $i_{1}$ and $i_{2}$. 
Now, for rounds $t \geq t^*+1$, the adversary assigns $r_{t}(i^*)=1$ and $c_{t}(i^*)=1$ to the optimal action $i^*$. For the suboptimal action $i\neq i^*$, the adversary assigns $r_{t^*+1}(i)=0$ and $c_{t^*+1}(i)=B^{\alpha}$ (since $c_{\max} = B^{\alpha}$, this is a valid  cost assignment), and $r_{t}(i)=c_{t}(i)=1$ for $t > t^*+1$.

Let $S_1$ be the case when $i^*=i_{1}$, and  $S_2$ be the case when $i^* = i_2$. For the first $t^*$ rounds,  any algorithm $\mathcal{A}$ would have the same behavior in both the cases $S_{1}$ and $S_{2}$.
Now,  at round $t^*+1$, assume that this algorithm $\mathcal{A}$ selects an action $i_{1}$ and $i_{2}$ with probability $p$ and $(1-p)$ respectively. Note that if the suboptimal action is chosen at round $t^*+1$, then the budget is depleted and the sum of the rewards is $0$. On the other hand, if $i^*$ is chosen at $t^*+1$, the algorithm receives a sum of $B^{\alpha}$ rewards in the end. Thus, if $i_{t^*+1}\neq i^*$, then the regret of the algorithm is $B^{\alpha}$. This implies that the expected regret of the algorithm is $0.5 pB^{\alpha} + 0.5(1-p)B^{\alpha} = B^{\alpha}/2$. The statement of the theorem follows.
\end{proof}
In the literature of BwK, the cost is always considered to be bounded above by a constant independent of the budget $B$. Here, we consider that the cost is bounded by a function of the budget $B$. Theorem \ref{thm:Cost} shows that the lower bound on the expected regret  scales at least linearly with the maximum cost $c_{max}$ in the adversarial BwK setup. If $\alpha> 1/2$, then it is impossible to achieve a regret bound of $O(\sqrt{B})$, which is order optimal in cases with small $c_{max}$.

In the adversarial BwK setup, the adversary can penalize the player in two ways. First, the adversary can control the reward of an action at any round. Second, the adversary can control the cost of an action, which is analogous to  penalizing the player on the number of rounds $T$. For $\alpha>1/2$, the latter penalty on the number of rounds $T$ becomes significant, and the minimum achievable regret is no longer $\Omega(\sqrt{B})$. In this setting with $\alpha>1/2$, the design of algorithms which achieve regret of $O(B^{\alpha})$ is left as   future work.

\section{Conclusion}
The study of BwK has been mostly focused on the stochastic regime. In this work, we considered the adversarial regime and proposed the order optimal algorithm EXP3.BwK for this setting. We also used ideas from the adversarial BwK setup  to design  EXP3++.BwK. This algorithm has an expected regret  of $O(\sqrt{KB\log (K)})$ and $O(\log^2(B))$  in the adversarial and stochastic  settings respectively. Thus, the algorithm is order optimal in the adversarial regime, and has an additional factor of $\log(B)$ in the stochastic regime. It is the first algorithm  that provides almost optimal performance guarantees in both stochastic and adversary   BwK settings. 
%Recently, the problem of designing algorithms that work well in both the stochastic and adversarial setting has gained attention\cite{pinto2017robust,dhillon2018stochastic}. 
As part of future work, we are considering designing an algorithm which achieves the optimal regret guarantees  with high probability in both the adversarial and the stochastic BwK settings. 

All the results in the literature of BwK assume that the maximum cost is bounded by a constant independent of $B$. We have shown that if the cost is $O(B^{\alpha})$, then the expected regret is at least $\Omega(B^{\alpha})$. Thus, the minimum expected regret scales at least linearly with the maximum cost of the BwK setup. This setting is of particular interest when $\alpha>1/2$ because the expected regret of $O(\sqrt{B})$, which is achievable in the setting where cost is bounded by a constant, becomes unachievable. Hence, there is a need to study this BwK setting, and design optimal algorithms whose expected regret is $O(B^{\alpha})$, which is left as a future work. 
\section{Appendix}
\subsection{Proof of Theorem 1}
\begin{proof}
Let $T=\max\{T(i^*),\tau(E)\}$, where 
\[i^*=\mbox{argmax}_{i\in[K]}\sum_{t=1}^{T(i)}\frac{r_{t}(i)}{c_{t}(i)}.\]
Additionally,
\begin{equation}\label{eq:loss1}
\begin{split}
    \sum_{i\in[K]}p_{t}(i)\hat{e}_t(i)&=p_{t}(i_t)\frac{r_{t}(i_t)}{p_{t}(i_t)\cdot c_{t}(i_t)}\\
    &=\frac{r_{t}(i_t)}{ c_{t}(i_t)},
\end{split}
\end{equation}
and
\begin{equation}\label{eq:loss2}
\begin{split}
\sum_{i\in[K]}p_{t}(i)\hat{e}_t(i)^2&=p_{t}(i_t)\frac{r_{t}(i_t)}{p_{t}(i_t)\cdot c_{t}(i_t)}\hat{e}_t(i_t)\\
&\stackrel{(a)}{\leq} \frac{\hat{e}_t(i_t)}{c_{min}}\\
&\stackrel{}{=} \frac{\sum_{i\in[K]}\hat{e}_t(i)}{c_{min}},
\end{split}
\end{equation}
where $(a)$ follows from the fact that for all $i\in[K]$, $r_{t}(i)/c_{t}(i)\leq 1/c_{min}$. Also, for all $i\in [K]$, we have
\begin{equation}\label{eq:unbiasedEstimate}
\begin{split}
       \mathbf{E}\Big[\hat{e}_t(i)|\{p_{t}(j)\}_{j\in[K]}\Big]&=p_{t}(i)\cdot\hat{e}_t(i)+(1-p_{t}(i))\cdot 0\\
       &=\frac{r_{t}(i)}{c_{t}(i)}.
\end{split}
\end{equation}

Since $W_{t}={\sum_{j\in [K]}w_{t}(j)}$, 
\begin{equation}\label{eq:Wratio}
\begin{split}
\frac{W_{t+1}}{W_{t}}
&=\sum_{i\in [K]}\frac{w_{t+1}(i)}{W_{t}}\\
                    &=\sum_{i\in [K]}\frac{w_{t}(i)\exp{(\gamma c_{min}\cdot\hat{e}_{t}(i)/K)}}{W_{t}}\\
                    &\stackrel{(a)}{=}\sum_{i\in [K]}\frac{p_{t}(i)-\gamma/K}{1-\gamma}\cdot\exp{(\gamma c_{min}\cdot\hat{e}_{t}(i)/K)}\\
                    &\stackrel{(b)}{\leq}\hspace{-3pt}\sum_{i\in [K]}\hspace{-3pt}\frac{p_{t}(i)-\gamma/K}{1-\gamma}\Bigg(\hspace{-3pt}1\hspace{-3pt}+\hspace{-3pt}\frac{\gamma c_{min}}{K}\hat{e}_{t}(i)\hspace{-3pt}\\
                    &\qquad+\hspace{-3pt}(e-2)\bigg(\hspace{-3pt}\frac{\gamma c_{min}}{K}\hat{e}_{t}(i)\hspace{-3pt}\bigg)^2\hspace{-2pt}\Bigg)\\
                    &\stackrel{(c)}{\leq}1+\frac{c_{min} \gamma/K}{(1-\gamma)}\sum_{i\in [K]}p_{t}(i)\hat{e}_{t}(i)\\
                   &\qquad +\frac{(e-2)c_{min}^2(\gamma/K)^2}{(1-\gamma)}\sum_{i\in [K]}p_{t}(i)\hat{e}_{t}(i)^2,
\end{split}
\end{equation}
where $(a)$ follows from the definition of $w_{t}(i)$, $(b)$ follows from the facts that for all $i\in[K]$, $p_{t}(i)>\gamma/K$ and for all $x\leq 1$, $e^{x}\leq 1+x+(e-2)x^2$, and $(c)$ follows from the fact that $\sum_{i\in[K]}p_{t}(i)=1$ and $\gamma/K>0$.

Now, taking logs on both sides of (\ref{eq:Wratio}), summing over $1,2,\ldots T+1$, and using $\log(1+x)\leq x$ for all $x>-1$, we get
 \begin{equation}\label{eq:ratio1}
 \begin{split}
       \log\frac{W_{T+1}}{W_{1}}&\leq \frac{c_{min}\gamma/K}{(1-\gamma)}\sum_{t=1}^{T}\sum_{i\in [K]}p_{t}(i)\hat{e}_{t}(i)\\
                    &\quad+\frac{(e-2)c_{min}^2(\gamma/K)^2}{(1-\gamma)}\sum_{t=1}^{T}\sum_{i\in [K]}p_{t}(i)\hat{e}_{t}(i)^2  .
 \end{split}
 \end{equation}
Additionally, for all $j\in[K]$, we have
\begin{equation}\label{eq:ratio2}
\begin{split}
     \log\frac{W_{T+1}}{W_{1}}&\geq \log\frac{w_{T+1}(j)}{W_1}\\
                                    &=\frac{c_{min}\gamma}{K}\sum_{t=1}^{T}\hat{e}_{t}(j)-\log(K).
\end{split}
\end{equation}
Combining (\ref{eq:ratio1}) and (\ref{eq:ratio2}), for all $j\in[K]$, we have 
\begin{equation}\label{eq:7}
\begin{split}
        &\frac{c_{min}\gamma}{K}\sum_{t=1}^{T}\hat{e}_{t}(j)-\log(K)\\&\leq \frac{c_{min}\gamma/K}{(1-\gamma)}\sum_{t=1}^{T}\frac{r_{t}(i_t)}{c_t(i_t)}+\\
        &\frac{(e-2)c_{min}^2(\gamma/K)^2}{c_{min}(1-\gamma)}\sum_{t=1}^{T}\sum_{i\in[K]}\hat{e}_{t}(i),
\end{split}
\end{equation}
where the right hand side of the above equation follows from (\ref{eq:loss1}) and (\ref{eq:loss2}). 
We will split the analysis into two cases: $T(i^*)\leq\tau(E)$ and $T(i^*)>\tau(E)$. For $T(i^*)\leq\tau(E)$, using (\ref{eq:7}), we have 
\begin{equation}\label{eq:adv1.1}
    \begin{split}
        &\frac{\gamma}{K}\sum_{t=1}^{T(i^*)}\hat{e}_{t}(i^*)-\frac{\log(K)}{c_{min}}\\
        &\leq \frac{\gamma/K}{(1-\gamma)}\sum_{t=1}^{\tau(E)}\frac{r_{t}(i_t)}{c_t(i_t)}+
        \frac{(e-2)(\gamma/K)^2}{(1-\gamma)}\sum_{t=1}^{\tau(E)}\sum_{i\in[K]}\hat{e}_{t}(i),
\end{split}
\end{equation}
where the inequality follows by replacing $T=\tau(E)$, and using the fact that $T(i^*)\leq\tau(E)$ and $\hat{e}_{t}(i^*)$ is non-negative.

Now, for $T(i^*)>\tau(E)$, using (\ref{eq:7}), we have
\begin{equation}\label{eq:adv1.2}
    \begin{split}
        &\frac{\gamma}{K}\sum_{t=1}^{T(i^*)}\hat{e}_{t}(i^*)-\frac{\log(K)}{c_{min}}\\
        &\leq \frac{\gamma/K}{(1-\gamma)}\sum_{t=1}^{T(i^*)}\frac{r_{t}(i_t)}{c_t(i_t)}+
        \frac{(e-2)(\gamma/K)^2}{(1-\gamma)}\sum_{t=1}^{T(i^*)}\sum_{i\in[K]}\hat{e}_{t}(i),\\
        &\stackrel{(a)}{=}     \frac{\gamma/K}{(1-\gamma)}\sum_{t=1}^{\tau(E)}\frac{r_{t}(i_t)}{c_t(i_t)}+
        \frac{(e-2)(\gamma/K)^2}{(1-\gamma)}\sum_{t=1}^{T(i^*)}\sum_{i\in[K]}\hat{e}_{t}(i),
\end{split}
\end{equation}
where $(a)$ follows from the fact that for all $t>\tau(E)$, $r_{t}(i_{t})/c_{t}(i_{t})=0$.  Therefore, (\ref{eq:adv1.2}) can be further simplified as 
\begin{equation}\label{eq:adv1.3}
    \begin{split}
        &\frac{\gamma}{K}\sum_{t=1}^{T(i^*)}\hat{e}_{t}(i^*)-\frac{\log(K)}{c_{min}}\\
        &\leq \frac{\gamma/K}{(1-\gamma)}\sum_{t=1}^{\tau(E)}\frac{r_{t}(i_t)}{c_t(i_t)}+\\
        &\frac{(e-2)(\gamma/K)^2}{(1-\gamma)}\Bigg(\sum_{t=1}^{\tau(E)}\sum_{i\in[K]}\hat{e}_{t}(i)+\sum_{t=\tau(E)+1}^{T(i^*)}\sum_{i\in[K]}\hat{e}_{t}(i)\Bigg).
\end{split}
\end{equation}
Combining (\ref{eq:adv1.1}) and (\ref{eq:adv1.3}), taking expectation  on both sides of the equation, and using (\ref{eq:unbiasedEstimate}), we have
\begin{equation}\label{eq:adv2}
    \begin{split}
      &\sum_{t=1}^{T{(i^*)}}\frac{r_{t}(i^*)}{c_{t}(i^*)}-\mathbf{E}\Bigg[\sum_{t=1}^{\tau(E)}\frac{r_{t}(i_t)}{c_t(i_t)}\Bigg]\\
      &\leq   \frac{K}{c_{min}\gamma}\log(K)+\gamma\sum_{t=1}^{T{(i^*)}}\frac{r_{t}(i^*)}{c_{t}(i^*)}\\
        &+\frac{(e-2)\gamma}{K}\mathbf{E}\Bigg[\sum_{t=1}^{\tau(E)}\sum_{i\in [K]} \frac{r_{t}(i)}{c_{t}(i)}\Bigg] \\
        &+ \frac{(e-2)\gamma}{K}\mathbf{P}(T(i^*)>\tau(E))\mathbf{E}\Bigg[ \sum_{t=\tau(E)+1}^{T(i^*)}\sum_{i\in[K]}\hat{e}_{t}(i)\Bigg] .
\end{split}
\end{equation}
Since $B(E)\geq B-K$, we have
$\abs{T(i^*)-\tau (E)}\leq K/c_{min}$.
Using $G(\mathcal{A^*})\leq B/c_{min}^2$ and $T(i^*)-\tau(E)\leq K/c_{min}$, we have
\begin{equation}
    \begin{split}
      \sum_{t=1}^{T{(i^*)}}\frac{r_{t}(i^*)}{c_{t}(i^*)}-\mathbf{E}\Bigg[\sum_{t=1}^{\tau(E)}\frac{r_{t}(i_t)}{c_t(i_t)}\Bigg]&\leq   \frac{K}{c_{min}\gamma}\log(K)+\gamma\cdot \frac{B}{c_{min}^2}\\
        &+{(e-2)\gamma}{}\cdot \Bigg(\frac{B}{c_{min}^2}+\frac{K}{c_{min}^2}\Bigg) .
\end{split}
\end{equation}

Using $\gamma=\sqrt{c_{min} K\log(K)/(B(e-1)+K(e-2))}$, the right hand side of the above equation is bounded by
\begin{equation}\label{eq:exp3.bwk}
     2\sqrt{\frac{((e-1)B+(e-2)K)K\log(K)}{c_{min}^3}}. 
\end{equation}
Since for all $t$, $c_{t}(i^*)\leq 1$, $T(i^*)\geq B$ and $B(E)\geq B-K$ . Also, $\tau(E)\leq B/c_{min}$. Thus,
\begin{equation}\label{eq:proof2}
    z(E)\leq 1. 
\end{equation}
Combining (\ref{eq:exp3.bwk}) and (\ref{eq:proof2}), the statement of the theorem follows.
\end{proof}

\subsection{Proof of Theorem 3}
\begin{proof}
Let $T=\max \{T(i^*),\tau(E)\}$. The proof of the theorem is split into following results. 

In Lemma \ref{lemma:UCB}, we show that for all $i\in [K]$, the efficiency $e(i)$ is
\[\mbox{LCB}_{t}(i)\leq e(i)\leq \mbox{UCB}_{t}(i),\]
with high probability as $t\to \infty$ (see Lemma \ref{lemma:UCB}) i.e. 
\begin{equation}
    \mathbf{P}(\mbox{UCB}_{t}(i)\leq e(i))\leq \frac{1}{Kt^{\alpha-1}},
\end{equation}
\begin{equation}
    \mathbf{P}(\mbox{LCB}_{t}(i)\geq e(i))\leq \frac{1}{Kt^{\alpha-1}}.
\end{equation}
This is used to show that  $\hat{\Delta}_{t}(i)\leq \Delta(i)$  with high probability as $t\to \infty$ (see Lemma \ref{lemma:gapEstimate}), i.e. 
\begin{equation}\label{eq:17}
    \mathbf{P}(\hat{\Delta}_{t}(i)\geq \Delta(i))\leq \frac{1}{t^{\alpha-1}}.
\end{equation}
Using Lemma \ref{lemma:exploration} and Lemma \ref{lemma:NumberOfRounds}, we show that (see Lemma \ref{lemma:gapLowerBound})
\begin{equation}
\begin{split}\label{eq:18}
    \mathbf{P}\Bigg(\hat{\Delta}_{t}(i)&\leq \frac{\Delta(i)}{2}\Bigg)\leq \Bigg(\frac{\log t}{tc_{min}^2\Delta(i)^2}\Bigg)^{\alpha - 2}+2\Bigg(\frac{1}{t}\Bigg)^{\frac{\beta c_{min}^2}{8}}\\
    &\qquad\qquad\qquad+\frac{2}{Kt^{\alpha -1}}.
\end{split}
\end{equation}
Thus, using (\ref{eq:17}) and (\ref{eq:18}), we have 
\[\frac{\Delta(i)}{2}\leq \hat{\Delta}_{t}(i)\leq \Delta(i),\]
with high probability as $t\to\infty$. 

Using Lemma \ref{lemma:martingale} and \ref{lemmma:gapAdv}, we have 
\begin{equation}\label{eq:12}
    \mathbf{P}\Bigg(
    \tilde{\Delta}_{t}(i)\leq \frac{t\Delta(i)}{2}\Bigg)\leq \Bigg(\frac{\log(t)}{tc_{min}^2\Delta(i)^2}\Bigg)^{\alpha-2} +\frac{1}{t},
\end{equation}
where $\tilde{\Delta}_{t}(i)=\sum_{n=1}^{t} (\hat{\ell}_{n}(i)-\hat{\ell}_{n}(i^*))$. Since 
\[p_{t}(i)\leq \exp(-\gamma_{t}\tilde{\Delta}_{t}(i)),\]
(\ref{eq:12}) is used to bound $\sum_{t=1}^{T}\mathbf{E}[p_{t}(i)]$, thus we have
\[\sum_{t=1}^{T}\mathbf{E}[p_{t}(i)]=O\Bigg(\frac{\log^2(B/c_{min})}{c_{min}^2\Delta(i)^2}\Bigg).\]
 Using the definition of $\epsilon_{t}(i)$ and (\ref{eq:12}), we have  
 \[\sum_{t=1}^{T}\mathbf{E}[\epsilon_{t}(i)]=O\Bigg(\frac{\log^2(B/c_{min})}{c_{min}^2\Delta(i)^2}\Bigg).\]
 Hence, the statement of the theorem follows. 
\end{proof}

\begin{lemma}\label{lemma:UCB}
For all $i\in[K]$ and $t\geq K$,
\begin{equation}
    \mathbf{P}(\mbox{UCB}_{t}(i)\leq e(i))\leq \frac{1}{Kt^{\alpha-1}},
\end{equation}
\begin{equation}
    \mathbf{P}(\mbox{LCB}_{t}(i)\geq e(i))\leq \frac{1}{Kt^{\alpha-1}},
\end{equation}
\end{lemma} 
\begin{proof}
 If $\mbox{UCB}_{t}(i)\leq e(i)$, then
\[\Bar{e}_{t}(i)+\frac{(1+1/\lambda) \eta_{t}(i)}{\lambda- \eta_{t}(i)}\leq e(i)=\frac{\mu(i)}{\rho(i)}.\]
Therefore, at least one of the events $U_{1}$ and $U_{2}$ is true, where
\[U_{1}:\Bar{r}_{t}(i)\leq \mu(i)-\eta_{t}(i),\]
\[U_{2}:\Bar{c}_{t}(i)\geq \rho(i)+\eta_{t}(i).\]
This can be proved by contradiction. Let both $U_{1}$ and $U_{2}$ are false. Then, we have
\begin{equation}
\begin{split}
\frac{\mu(i)}{\rho(i)}-\frac{\Bar{r}_{t}(i)}{\Bar{c}_{t}(i)}&= \frac{\mu(i)\Bar{c}_{t}(i)-\rho(i)\Bar{r}_{t}(i)}{\rho(i)\Bar{c}_{t}(i)},\\
&=\frac{\mu(i)(\Bar{c}_{t}(i)-\rho(i))+\rho(i)(\mu(i)-\Bar{r}_{t}(i))}{\rho(i)\Bar{c}_{t}(i)},\\
&\stackrel{(a)}{\leq} \frac{\mu(i)\eta_{t}(i)+\rho(i)\eta_{t}(i)}{\rho(i)\Bar{c}_{t}(i)},\\
&\stackrel{(b)}{\leq}\frac{\eta_{t}(i)}{\lambda(\lambda-\eta_{t}(i))}+\frac{\eta_{t}(i)}{\lambda-\eta_{t}(i)},\\
&= \frac{(1+1/\lambda) \eta_{t}(i)}{\lambda- \eta_{t}(i)},
\end{split}
\end{equation}
where $(a)$ follows from the fact that both $U_{1}$ and $U_{2}$ are false, and $(b)$ follows from the fact that  $U_{1}$ and $U_{2}$ are false, and $\lambda\leq c_{min}$. Hence, at least one of the events $U_{1}$ and $U_{2}$ is true. 
Now, using Hoeffding's inequality, we have
\begin{equation}\label{eq:U1}
    \mathbf{P}(U_{1})\leq \frac{1}{Kt^{\alpha}},
\end{equation}
and
\begin{equation}\label{eq:U2}
    \mathbf{P}(U_{2})\leq \frac{1}{Kt^{\alpha}}.
\end{equation}
Thus,
\begin{equation}
\begin{split}
        \mathbf{P}(\mbox{UCB}_{t}(i)\leq e(i))&\leq  \mathbf{P}(U_{1})+ \mathbf{P}(U_{2})\\
        &\leq \frac{1}{Kt^{\alpha-1}}.
\end{split}
\end{equation}
Similarly, if $\mbox{LCB}_{t}(i)\geq e(i)$, then 
\[\Bar{e}_{t}(i)-\frac{(1+1/\lambda) \eta_{t}(i)}{\lambda- \eta_{t}(i)}\geq e(i)=\frac{\mu(i)}{\rho(i)}.\]
Therefore, at least one of the events $L_{1}$ and $L_{2}$ is true, where
\[L_{1}:\Bar{r}_{t}(i)\geq \mu(i)+\eta_{t}(i),\]
\[L_{2}:\Bar{c}_{t}(i)\leq \rho(i)-\eta_{t}(i).\]
This can be proved by contradiction. Now, using Hoeffding's inequality, we have
\begin{equation}\label{eq:L1}
    \mathbf{P}(L_{1})\leq \frac{1}{Kt^{\alpha}},
\end{equation}
and
\begin{equation}\label{eq:L2}
    \mathbf{P}(L_{2})\leq \frac{1}{Kt^{\alpha}}.
\end{equation}
Thus,
\begin{equation}
\begin{split}
        \mathbf{P}(\mbox{LCB}_{t}(i)\geq e(i))&\leq  \mathbf{P}(L_{1})+ \mathbf{P}(L_{2})\\
        &\leq \frac{1}{Kt^{\alpha-1}}.
\end{split}
\end{equation}
Hence proved.
\end{proof}
\begin{lemma}\label{lemma:gapEstimate}
For all $i\in[K]$ and $t\geq K$,
\begin{equation}
    \mathbf{P}(\hat{\Delta}_{t}(i)\geq \Delta(i))\leq \frac{1}{t^{\alpha-1}},
\end{equation}
\end{lemma}
\begin{proof}
 Since $\Delta(i)=\max_{j\in[K]}e(j)-e(i)$, we have
\begin{equation}
\begin{split}
    \mathbf{P}(\hat{\Delta}_{t}(i)\geq \Delta(i))&=  \mathbf{P}(\max_{j\neq i}\mbox{LCB}_{t}(j)-\mbox{UCB}_{t}(i)\geq \Delta(i))\\
    &\leq \sum_{j\neq i}\mathbf{P}(\mbox{LCB}_{t}(j)\geq e(j))\\
    &\qquad\qquad+\mathbf{P}(\mbox{UCB}_{t}(i)\leq e(i))\\
    &\leq \frac{1}{t^{\alpha-1}},
\end{split}
\end{equation}
where the last inequality follows from Lemma \ref{lemma:UCB}. Hence proved.
\end{proof}

\begin{lemma}\label{lemma:exploration}
For all $i\in[K]$, let
\[t_{min}(i)=\min\{t:t\geq 4K\beta(\log t)^2/\Delta(i)^4\log(K)\}.\]
We define two events $A(i,t)$ and $A(i^*,i,t)$ as
\begin{equation}
    A(i,t)=\Bigg\{\mbox{there exists an } n\leq t:\epsilon_{n}(i)<\frac{\beta \log t}{t\Delta(i)^2}\Bigg\},
\end{equation}
\begin{equation}
    A(i^*,i,t)\hspace{-3pt}=\hspace{-3pt}\Bigg\{\hspace{-3pt}\mbox{there exists an } n\leq t:\epsilon_{n}(i^*)<\frac{\beta \log t}{t\Delta(i)^2}\Bigg\}.
\end{equation}
For $t>t_{min}(i)$ and $\alpha\geq 3$, we have 
\begin{equation}
    \mathbf{P}(A(i,t))\leq \frac{1}{2}\Bigg(\frac{\log t}{tc_{min}^2\Delta(i)^2}\Bigg)^{\alpha - 2},
\end{equation}
\begin{equation}
        \mathbf{P}(A(i^*, i,t))\leq \frac{1}{2}\Bigg(\frac{\log t}{tc_{min}^2\Delta(i)^2}\Bigg)^{\alpha - 2}.
\end{equation}
\end{lemma}
\begin{proof}
We start with proving the bound on the probability of the event $A(i,t)$. This proof is divided into two parts. First, for $n\leq tc_{min}^2\Delta(i)^2/\log(t)$, using the Lemma \ref{lemma:gapEstimate}, we show that $A(i,t)$ does not occur with high probability as $t\to \infty$ . Later, for $n\geq tc_{min}^2\Delta(i)^2/\log(t)$, we bound the probability of the event $A(i,t)$ using the Lemma \ref{lemma:gapEstimate}. 

For $n\leq tc_{min}^2\Delta(i)^2/\log(t)$, we have
\begin{equation}\label{eq:bound1}
\begin{split}
\frac{\beta \log(n)}{n\hat{\Delta}^2_{n}(i)}&\stackrel{(a)}{\geq} \frac{\beta c_{min}^2 \log(n)}{n},\\
&\stackrel{(b)}{\geq} \frac{\beta \log(n)\log(t)}{t{\Delta}(i)^2},\\
&\geq \frac{\beta\log(t)}{t{\Delta}(i)^2},
\end{split}
\end{equation}
where $(a)$ follows from the definition of $\hat{\Delta}_{n}(i)$, and $(b)$ follows from the range of $n$. For $t\geq t_{min}$, 
\begin{equation}\label{eq:bound2}
    0.5\sqrt{\frac{\log(K)}{tK}}\geq \frac{\beta \log(t)}{t\Delta(i)^2}.
\end{equation}
Additionally, using Lemma \ref{lemma:gapEstimate}, $\hat{\Delta}_{n}(i)\leq \Delta(i)$ w.h.p as $n\to \infty$. Therefore, combining  (\ref{eq:bound2}), $\hat{\Delta}_{n}(i)\leq \Delta(i)$ and (\ref{eq:bound1}), we have
\begin{equation}\label{eq:bound3}
\begin{split}
      \epsilon_{n}(i)&\geq\frac{\beta \log t}{t\Delta(i)^2}.
\end{split}
\end{equation}

Now, for $n\geq tc_{min}^2\Delta(i)^2/\log(t)$, we have
\begin{equation}
\begin{split}
 &\mathbf{P}\Bigg(\mbox{There exists } n\in \Bigg[\frac{tc_{min}^2\Delta(i)^2}{\log(t)},t\Bigg]: \epsilon_{n}(i)<\frac{\beta \log t}{t\Delta(i)^2} \Bigg)\\
 &=\mathbf{P}\Bigg(\mbox{There exists } n\in \Bigg[\frac{tc_{min}^2\Delta(i)^2}{\log(t)},t\Bigg]: \hat{\Delta}_{n}(i)\geq \Delta(i) \Bigg)\\
 &\leq \sum_{n=\frac{tc_{min}^2\Delta(i)^2}{\log(t)} }^t \frac{1}{n^{\alpha-1}}\leq \frac{1}{2}\Bigg(\frac{\log t}{tc_{min}^2\Delta(i)^2}\Bigg)^{\alpha - 2}.
\end{split}
\end{equation}
Similarly, we can bound the probability of $\mathbf{P}(A(i^*, i,t))$ by using the fact that $\Delta(i^*)=0<\Delta(i)$ for $i\neq i^*$. Hence proved.
\end{proof}

\begin{lemma}\label{lemma:NumberOfRounds}
For all $i\in[K]$ and $t\geq t_{min}(i)$, we have
\begin{equation}
    \mathbf{P}\Bigg(N_{t}(i)\leq \frac{\beta \log t}{2\Delta(i)^2}\Bigg)\leq \hspace{-3pt}\Bigg(\frac{1}{t}\Bigg)^{\frac{\beta c_{min}^2}{8}}\hspace{-3pt}+\frac{1}{2}\Bigg(\frac{\log t}{tc_{min}^2\Delta(i)^2}\Bigg)^{\alpha - 2}\hspace{-6pt}.
\end{equation}
Additionally,
\begin{equation}\label{eq:second}
    \mathbf{P}\Bigg(N_{t}(i^*)\leq \frac{\beta \log t}{2\Delta(i)^2}\Bigg)\leq \hspace{-3pt}\Bigg(\frac{1}{t}\Bigg)^{\frac{\beta c_{min}^2}{8}}\hspace{-3pt}+\frac{1}{2}\Bigg(\frac{\log t}{tc_{min}^2\Delta(i)^2}\Bigg)^{\alpha - 2}\hspace{-6pt}.
\end{equation}
\end{lemma}
\begin{proof}
 We have
\begin{equation}
\begin{split}
   &\mathbf{P}\Bigg(N_{t}(i)\leq \frac{\beta \log t}{2\Delta(i)^2}\Bigg)\\
   &\leq \mathbf{P}\Bigg(A^C(i,t) \mbox{ and } N_{t}(i)\leq \frac{\beta \log t}{2\Delta(i)^2}\Bigg)+\mathbf{P}\Bigg(A(i,t)\Bigg),\\
   &\stackrel{(a)}{\leq}\exp\Bigg(\frac{-\beta \log t}{8\Delta(i)^2}\Bigg)+\frac{1}{2}\Bigg(\frac{\log t}{tc_{min}^2\Delta(i)^2}\Bigg)^{\alpha - 2},\\
   &\stackrel{(b)}{\leq} \Bigg(\frac{1}{t}\Bigg)^{\frac{\beta c_{min}^2}{8}}\hspace{-3pt}+\frac{1}{2}\Bigg(\frac{\log t}{tc_{min}^2\Delta(i)^2}\Bigg)^{\alpha - 2},
\end{split}
\end{equation}
where $A^C(i,t)$ is the complement of the event $A(i,t)$, $(a)$ follows from the Theorem 8 in \cite{seldin2017improved} and Lemma \ref{lemma:exploration}, and $(b)$ follows from the fact that for all $i\in[K]$, $\Delta(i)\leq 1/c_{min}^2$. Similarly, we can bound the probability in (\ref{eq:second}).
\end{proof}

\begin{lemma}\label{lemma:gapLowerBound}
For all $i\in[K]$, $t\geq t_{min}(i)$, $\alpha\geq 3$ $\beta \geq 64( \alpha +1)/c_{min}^2\geq 256/c_{min}^2$, we have 
\begin{equation}
\begin{split}
    \mathbf{P}\Bigg(\hat{\Delta}_{t}(i)&\leq \frac{\Delta(i)}{2}\Bigg)\leq \Bigg(\frac{\log t}{tc_{min}^2\Delta(i)^2}\Bigg)^{\alpha - 2}+2\Bigg(\frac{1}{t}\Bigg)^{\frac{\beta c_{min}^2}{8}}\\
    &\qquad\qquad\qquad+\frac{2}{Kt^{\alpha -1}}.
\end{split}
\end{equation}
\end{lemma}
\begin{proof}
Using Lemma \ref{lemma:UCB}, we have 
\begin{equation}\label{eq:prob1}
\begin{split}
    &\mathbf{P}\big((\mbox{UCB}_{t}(i^*)\leq e(i^*))\mbox{ or }(\mbox{LCB}_{t}(i)\geq e(i))\big)\\
    &\leq {2}/{Kt^{\alpha-1}}.
\end{split}
\end{equation}
Now, assume $\mbox{UCB}_{t}(i^*)\geq e(i^*)$ and $\mbox{LCB}_{t}(i)\leq e(i)$, we have
\begin{equation}\label{eq:lower1}
\begin{split}
\hat{\Delta}_{t}(i) &\geq \max_{j\neq i}\mbox{LCB}_{t}(j)-\mbox{UCB}_{t}(i),\\
&\geq \mbox{LCB}_{t}(i^*)-\mbox{UCB}_{t}(i),\\
&=\Bar{e}_{t}(i^*)-\eta_{t}(i^*)-\Bar{e}_{t}(i)-\eta_{t}(i)\\
&\geq e(i^*)-2\eta_{t}(i^*)-e(i)-2\eta_{t}(i)\\
&=\Delta(i)-2\eta_{t}(i^*)-2\eta_{t}(i).
\end{split}
\end{equation}
Similarly, using Lemma \ref{lemma:NumberOfRounds}, we have
\begin{equation}\label{eq:prob2}
\begin{split}
 &\mathbf{P}\Bigg(N_{t}(i)\leq \frac{\beta \log t}{2\Delta(i)^2}\mbox{ or } N_{t}(i^*)\leq \frac{\beta \log t}{2\Delta(i)^2}\Bigg)\\
   &\leq 2\Bigg(\frac{1}{t}\Bigg)^{\frac{\beta c_{min}^2}{8}}+ \Bigg(\frac{\log t}{tc_{min}^2\Delta(i)^2}\Bigg)^{\alpha - 2}.
\end{split}
\end{equation}
Now, assuming $N_{t}(i)> {\beta \log t}/{2\Delta(i)^2}$ and $N_{t}(i^*)>{\beta \log t}/{2\Delta(i)^2}$, we have 
\begin{equation}\label{eq:lower2}
\begin{split}
    \hat{\Delta}_{t}(i) &\geq\Delta(i)-2\eta_{t}(i^*)-2\eta_{t}(i),\\
    &\geq \Delta(i) -4\sqrt{\frac{2\Delta(i)^2\alpha\log(tK^{1/\alpha})}{2\beta\log(t)}},\\
    &\geq \Delta(i)\Bigg(1-4\sqrt{\frac{\alpha+1}{c_{min}^2\beta}}\Bigg),\\
    &\geq \Delta/2.
\end{split}
\end{equation}
Therefore, combining (\ref{eq:prob1}),(\ref{eq:lower1}),(\ref{eq:prob2}) and (\ref{eq:lower2}), the statement of the theorem follows. Hence proved.
\end{proof}
\begin{lemma}\label{lemma:martingale}For all $i\in[K]$, let $X_{t}(i)=\Delta(i)-(\hat{\ell}_{t}(i)-\hat{\ell}_{t}(i^*))$ be the martingale difference sequence with respect to filtration $\mathcal{F}_{1},\ldots,\mathcal{F}_{1}$ where $\mathcal{F}_{t}$ is the sigma field based on all the past actions, their rewards and their costs until round $t$. Then, for $t\geq t_{min}(i)$,we have
\begin{equation}
    \mathbf{P}\Bigg(\max_{1\leq n\leq t}X_{n}(i)\geq \frac{1.25t\Delta(i)^2}{c_{min}\beta\log(t)}\Bigg)\leq \frac{1}{2}\Bigg(\frac{\log t}{tc_{min}^2\Delta(i)^2}\Bigg)^{\alpha-2},
\end{equation}
\begin{equation}
    \mathbf{P}\Bigg(\nu_{t}(i)\geq \frac{2t^2\Delta(i)^2}{c_{min}^3\beta\log(t)}\Bigg)\leq \Bigg(\frac{\log t}{tc_{min}^2\Delta(i)^2}\Bigg)^{\alpha-2},
\end{equation}
where $\nu_{t}(i)=\sum_{n=1}^{t}\mathbf{E}[X_{n}(i)^2|\mathcal{F}_{n-1}]$.
\end{lemma}
\begin{proof}
We bound the magnitude of $X_{n}(i)$. For all $i\in[K]$, we have
\begin{equation}\label{eq:limit1}
\begin{split}
X_{n}(i)&=\Delta(i)-(\hat{\ell}_{n}(i)-\hat{\ell}_{n}(i^*)),\\
&\leq \frac{1}{c_{min}}+ \hat{\ell}_{n}(i^*),\\
&\leq \frac{1}{c_{min}} +\frac{1}{c_{min}\epsilon_{n}(i^*)},\\
&\leq \frac{1}{c_{min}}\Bigg(1 +\max\Bigg\{2K,2\sqrt{\frac{nK}{\log(K)}},{\frac{n\hat{\Delta}_{n}(i^*)^2}{\beta \log(n)}}\Bigg\}\Bigg),\\
&\leq\frac{1.25}{c_{min}}\max\Bigg\{2K,2\sqrt{\frac{nK}{\log(K)}},{\frac{n\hat{\Delta}_{n}(i^*)^2}{\beta \log(n)}}\Bigg\}.
\end{split}
\end{equation}
Similar to the proof of Lemma \ref{lemma:exploration}, for $t\geq t_{min}$ and $n\leq tc_{min}^2\Delta(i)^2/\log(t)$, we have $\epsilon_{n}(i^*)\geq t\Delta(i)^2/\beta\log(t) $ and (see (\ref{eq:bound1}))
\begin{equation}
\begin{split}
\frac{\beta \log(n)}{n\hat{\Delta}^2_{n}(i)}
&\geq \frac{\beta\log(t)}{t{\Delta}(i)^2}.
\end{split}
\end{equation} 
Additionally, for $t\geq t_{min}$, 
\begin{equation}
    0.5\sqrt{\frac{\log(K)}{tK}}\geq \frac{\beta \log(t)}{t\Delta(i)^2},
\end{equation}
and using Lemma \ref{lemma:gapEstimate}, $\hat{\Delta}_{n}(i)\leq \Delta(i)$ w.h.p as $n\to \infty$. Therefore, using for all $i\in[K]$ $\Delta(i^*)=0\leq \Delta(i)$,
for $t_{1}\leq tc_{min}^2\Delta(i)^2/\log(t)$ and $t\geq t_{min}(i)$,
\begin{equation}\label{eq:inf1}
    \max_{1\leq n\leq t_{1}}X_{n}(i)\leq  \frac{1.25t\Delta(i)^2}{c_{min}\beta\log(t)},
\end{equation}
w.h.p at $t_{1}\to\infty$. Now,
\begin{equation}\label{eq:maxMartingale}
\begin{split}
    &\mathbf{P}\Bigg(\max_{1\leq n\leq t}X_{n}(i)\geq  \frac{1.25t\Delta(i)^2}{c_{min}\beta\log(t)}\Bigg)\\
    &\stackrel{(a)}{=}\mathbf{P}\Bigg(\exists n\in\Bigg[\frac{tc_{min}^2\Delta(i)^2}{\log(t)},t\Bigg]: X_{n}(i)\geq  \frac{1.25t\Delta(i)^2}{c_{min}\beta\log(t)}\Bigg),\\
    &\stackrel{(b)}{\leq}\mathbf{P}\Bigg(\exists n\in\Bigg[\frac{tc_{min}^2\Delta(i)^2}{\log(t)},t\Bigg]: \hat{\Delta}_{n}(i)\geq  \Delta(i)\Bigg),\\
    &\stackrel{(c)}{\leq}\frac{1}{2}\Bigg(\frac{\log(t)}{tc_{min}^2\Delta(i)^2}\Bigg)^{\alpha-2},
\end{split}
\end{equation}
where $(a)$ follows from (\ref{eq:inf1}), $(b)$ follows from (\ref{eq:limit1}), and $(c)$ follows from Lemma \ref{lemma:gapEstimate}. 

Now, we bound $\nu_{t}(i)=\sum_{n=1}^{t}\mathbf{E}[X_{n}(i)^2|\mathcal{F}_{n-1}]$. For all $i\in[K]$, we have
\begin{equation}
\begin{split}
    &\mathbf{E}[X_{n}(i)^2|\mathcal{F}_{n-1}]\\
    &\leq \mathbf{E}[(\hat{\ell}_{n}(i^*)-\hat{\ell}_{n}(i))^2|\mathcal{F}_{n-1}],\\
    &\stackrel{(a)}{=}\mathbf{E}[\hat{\ell}_{n}(i^*)^2|\mathcal{F}_{n-1}]+\mathbf{E}[\hat{\ell}_{n}(i)^2|\mathcal{F}_{n-1}],\\
    &=\Tilde{p}_{n}(i)\Bigg(\frac{\ell_{n}(i)}{\Tilde{p}_{n}(i)}\Bigg)^2+\Tilde{p}_{n}(i^*)\Bigg(\frac{\ell_{n}(i^*)}{\Tilde{p}_{n}(i^*)}\Bigg)^2,\\
    &\leq \frac{1}{c_{min}^2\Tilde{p}_{n}(i)}+\frac{1}{c_{min}^2\Tilde{p}_{n}(i^*)},\\
    &\stackrel{(b)}{\leq}\frac{1}{c_{min}^3}\Bigg(\max\Bigg\{2K,2\sqrt{\frac{nK}{\log(K)}},{\frac{n\hat{\Delta}_{n}(i)^2}{\beta \log(n)}}\Bigg\}\Bigg)+\\
    & \qquad\qquad\max\Bigg\{2K,2\sqrt{\frac{nK}{\log(K)}},{\frac{n\hat{\Delta}_{n}(i^*)^2}{\beta \log(n)}}\Bigg\}\Bigg),\\
\end{split}
\end{equation}
where $(a)$ follows from the fact that for all $i\in[K]$ and $n\leq t$, $\hat{\ell}_{n}(i^*)\cdot\hat{\ell}_{n}(i))=0$, and $(b)$ follows from (\ref{eq:limit1}). 

Similar to (\ref{eq:maxMartingale}), we bound the $\nu_{t}(i)$ as follows
\begin{equation}
\begin{split}
&\mathbf{P}\Bigg(\nu_{t}(i)\geq \frac{2t^2\Delta(a)^2}{c_{min}^3\beta\log(t)}\Bigg)\\
&\stackrel{(a)}{\leq} \mathbf{P}\Bigg(\exists n\in\Bigg[\frac{tc_{min}^2\Delta(i)^2}{\log(t)},t\Bigg]: \hat{\Delta}_{n}(i)\geq  \Delta(i)\Bigg)\\
&\qquad + \mathbf{P}\Bigg(\exists n\in\Bigg[\frac{tc_{min}^2\Delta(i)^2}{\log(t)},t\Bigg]: \hat{\Delta}_{n}(i^*)\geq  0\Bigg),\\
&\stackrel{(b)}{\leq} \Bigg(\frac{\log(t)}{tc_{min}^2\Delta(i)^2}\Bigg)^{\alpha-2},
\end{split}
\end{equation}
where $(a)$ can be implied in a similar way as $(b)$ of (\ref{eq:maxMartingale}), and $(b)$ follows from Lemma \ref{lemma:gapEstimate}. 
\end{proof}

\begin{lemma}\label{lemmma:gapAdv}
For all $t\geq t_{min}(i)$ and $\beta \geq 256/c_{min}^2$, we have
\begin{equation}
    \mathbf{P}\Bigg(
    \tilde{\Delta}_{t}(i)\leq \frac{t\Delta(i)}{2}\Bigg)\leq \Bigg(\frac{\log(t)}{tc_{min}^2\Delta(i)^2}\Bigg)^{\alpha-2} +\frac{1}{t}.
\end{equation}
where $\tilde{\Delta}_{t}(i)=\sum_{n=1}^{t} (\hat{\ell}_{n}(i)-\hat{\ell}_{n}(i^*))$.
\end{lemma}
\begin{proof}
We have
\begin{equation}
\begin{split}
   & \mathbf{P}\Bigg(\tilde{\Delta}_{t}(i)\leq \frac{t\Delta(i)}{2}\Bigg)\\
   &=\mathbf{P}\Bigg(t\Delta(i)-\tilde{\Delta}_{t}(i)\geq \frac{t\Delta(i)}{2}\Bigg),\\
   &\leq \mathbf{P}(M_{1}(t))+\mathbf{P}(M_{2}(t))+\mathbf{P}(M_{3}(t)),
\end{split}
\end{equation}
where 
\[M_{1}(t)=\Bigg\{\max_{1\leq n\leq t}X_{n}(i)\geq  \frac{1.25t\Delta(i)^2}{c_{min}\beta\log(t)}\Bigg\},\]
\[M_{2}(t)=\Bigg\{\nu_{t}(i)\geq \frac{2t^2\Delta(a)^2}{c_{min}^2\beta\log(t)}\Bigg\},\]
and
\[M_{3}(t)\hspace{-2pt}=\hspace{-3pt}\Bigg\{t\Delta(i)-\tilde{\Delta}_{t}(i)\geq \frac{t\Delta(i)}{2} \mbox{ and } M_{1}(t) \mbox{ and } M_{2}(t)\hspace{-3pt}\Bigg\}.\]

The probability of the events $M_{1}(t)$ and $M_{2}(t)$ can be bound using Lemma \ref{lemma:martingale}, and using the fact that $c_{min}\leq 1$. 

Let $w_1={2t^2\Delta(a)^2}/{c_{min}^2\beta\log(t)}$, $w_{2}={1.25t\Delta(i)^2}/{c_{min}\beta\log(t)}$, and $w_3 =1/t$. For all $t\geq t_{min}(i)$ and $\beta \geq 256/c_{min}^2$, we have
\begin{equation}\label{eq:bernstien}
\begin{split}
&\sqrt{2w_{1}\log\frac{1}{w_3}}+\frac{w_{2}}{3}\log\frac{1}{w_3}\\
&=\sqrt{\frac{4t^2\Delta(i)^2\log t}{c_{min}^2\beta \log t}}+ \frac{1.25t\Delta(i)^2\log t}{c_{min}\beta \log t},\\
&\leq t\Delta(i)\Bigg(\frac{2t}{\sqrt{c_{min}^2\beta}}+\frac{1.25}{3c_{min}^2\beta}\Bigg),\\
&\leq \frac{1}{2}t\Delta(i).
\end{split}
\end{equation}
Thus, using Bernstein's inequality for martingales and (\ref{eq:bernstien}), we can bound the probability of $M_{3}(t)$ as follows
\begin{equation}
    \mathbf{P}(M_{3}(t))\leq \frac{1}{t}.
\end{equation}
Thus, combining the bounds over the probabilities of the events $M_{1}(t)$, $M_{2}(t)$ and $M_{3}(t)$, the statement of the lemma follows.
\begin{lemma}\label{lemma:epsilonBound}For all $i\in[K]$, $\tau(E)\geq t_{min}(i)$,$T=\max\{\tau(E),T(i^*)\}$, $\alpha= 3$ and $\beta= 256/c_{min}^2$, we have
\begin{equation}
\begin{split}
  \sum_{t=1}^{T}\mathbf{E}[\epsilon_{t}(i)] &\leq    t_{min}(i)+ \frac{4\beta(\log^2(T)+\log(T))}{\Delta(i)^2}+\\
& \frac{\log^2(T)+\log(T)}{c_{min}^2\Delta(i)^2}+\frac{2}{K}(\log(T)+1)\\
&+\frac{2\pi ^2}{3}.
\end{split}
\end{equation}
\end{lemma}
\textit{Proof:} We have 
\begin{equation}
\begin{split}
&\sum_{t=1}^{T}\mathbf{E}[\epsilon_{t}(i)]\\
&=\sum_{t=1}^{T}\mathbf{E}\Bigg[\min\Bigg\{\frac{1}{2K},\frac{1}{2}\sqrt{\frac{\log(t)}{tK}},\frac{\beta\log t}{t\hat{\Delta}_{t}(i)^2}\Bigg\}\Bigg],\\
&\leq \sum_{t=1}^{T}\mathbf{E}\Bigg[\frac{\beta\log t}{t\hat{\Delta}_{t}(i)^2}\Bigg],\\
&\stackrel{(a)}{\leq} t_{min}(i)+ \frac{4\beta(\log^2(T)+\log(T))}{\Delta(i)^2}+ \\
&\sum_{t=t_{min}(i)}^{T}\Bigg(\Bigg(\frac{\log t}{tc_{min}^2\Delta(i)^2}\Bigg)^{\alpha - 2} +\frac{2}{Kt^{\alpha -1}}+2\Bigg(\frac{1}{t}\Bigg)^{\frac{\beta c_{min}^2}{8}}\Bigg),\\
&\leq t_{min}(i)+ \frac{4\beta(\log^2(T)+\log(T))}{\Delta(i)^2}+\\
& \frac{\log^2(T)+\log(T)}{c_{min}^2\Delta(i)^2}+\frac{2}{K}(\log(T)+1)+\frac{2\pi ^2}{3},
\end{split}
\end{equation}
where $(a)$ follows from Lemme \ref{lemma:gapLowerBound}.
\end{proof}
%\begin{lemma}\label{lemma:epsilonBound}For all $i\in[K]$, $\alpha= 3$ and $\beta= 256/c_{min}^2$, there exists a constant $m$ such that
%\begin{equation}
%\begin{split}
%  \sum_{t=1}^{\tau(E)}\mathbf{E}[\epsilon_{t}(i)] &\geq      m \log^2(\tau(E)).
%\end{split}
%\end{equation}
%\end{lemma}
%\textit{Proof:} We have 
%\begin{equation}
%\begin{split}
%&\sum_{t=1}^{\tau(E)}\mathbf{E}[\epsilon_{t}(i)] \\
%&\stackrel{(a)}{\geq}    \sum_{t=t_{min}(i)}^{\tau(E)}\frac{4\beta\log t}{t{\Delta}(i)^2}\Bigg(1-\Bigg(\frac{\log t}{tc_{min}^2\Delta(i)^2}\Bigg)^{\alpha - 2}\\ &\qquad-\frac{2}{Kt^{\alpha -1}}-2\Bigg(\frac{1}{t}\Bigg)^{\frac{\beta c_{min}^2}{8}}\Bigg),\\
%&\stackrel{(b)}{\geq} \sum_{t=t_{min}(i)}^{\tau(E)}\frac{1024\log t}{t}\Bigg(1-4\Bigg(\frac{\log t}{t}\Bigg)\Bigg),\\
%&\stackrel{(c)}{\geq} m \log^2(\tau(E)),
%\end{split}
%\end{equation}
%where $(a)$ follows from the Lemma \ref{lemma:gapLowerBound}, $(b)$ follows by replacing the values of $\alpha$, $\beta$, and assuming $\log(B)>2/K$, and $(c)$ follows from the fact that $(1-4\log(t)/t)$ is a fraction,  $t_{min}(i)=O(\log(\log(B)))$ and the limit on the summation using integration.

\begin{lemma}\label{lemma:pBound} For all $i\in[K]$, $\tau(E)\geq t_{min}(i)$, $\gamma \geq c^2_{min}\sqrt{K\log(K)/B(1+(e-2)/c_{min}^2)}$ and $\alpha\geq 3$, there exists a constant $m_{2}$ 
\begin{equation}
\begin{split}
        \sum_{t=1}^{T}\mathbf{E}[p_{t}(i)]&\leq t_{min}(i)+m_{2}\frac{\log^2(T)}{c_{min}^2\Delta(i)^2}.
\end{split}
\end{equation}
\end{lemma}
\begin{proof}
 We have
\begin{equation}
\begin{split}
    &\sum_{t=1}^{T}\mathbf{E}[p_{t}(i)]\\
    &\leq \sum_{t=1}^{T}\mathbf{E}\Big[\exp(-\gamma_{t}\tilde{\Delta}_{t}(i))\Big]  ,\\
    &\stackrel{(a)}{\leq} t_{min}(i)+\hspace{-3pt}\sum_{t=t_{min}(i)}^{T}\hspace{-3pt}\Bigg[\hspace{-2pt}e^{-\sqrt{\frac{\log(K)}{tK}}\frac{ t{\Delta}(i)}{4K}}+\hspace{-3pt}\frac{1}{t}\\
    &\qquad+\hspace{-3pt}\Bigg(\frac{\log(t)}{tc_{min}^2\Delta(i)^2}\Bigg)^{\alpha-2}+\Bigg(\frac{\log t}{tc_{min}^2\Delta(i)^2}\Bigg)^{\alpha - 2}\\
    &\qquad+\frac{2}{Kt^{\alpha -1}}+2\Bigg(\frac{1}{t}\Bigg)^{\frac{\beta c_{min}^2}{8}}\Bigg],\\
    &\stackrel{(b)}{\leq}t_{min}(i)+O\Bigg(\frac{\log^2(T)}{c_{min}^2\Delta(i)^2}\Bigg),\\
\end{split}
\end{equation}
where $(a)$ follows from the Lemma \ref{lemma:gapLowerBound} , and $(b)$ follows from bounds over the summation of sequences via integration. 
\end{proof}
\subsection{Proof of Theorem 4}
\begin{proof}
For all $i\in[K]$, 
\begin{equation}
 p_{t}(i)=\frac{\exp(-\gamma_{t}\sum_{n=1}^{t-1}\hat{\ell}_{n}(i))}{\sum_{i\in[K]}\exp(-\gamma_{t}\sum_{n=1}^{t-1}\hat{\ell}_{n}(i))},
\end{equation}
and $\gamma_{t}=0.5\sqrt{c_{min}^2\log(K)/Kt}$. Therefore, using Lemma 7 of \cite{seldin2014one}, we have
\begin{equation}
\begin{split}
&\sum_{t=1}^{T}\sum_{i\in [K]} p_{t}(i) \hat{\ell}_{t}(i) -\min_{j\in[K]}\sum_{t=1}^{T}\hat{\ell}_{t}(j)\\
&\leq \frac{1}{2}\sum_{t=1}^{T}\gamma_{t}\sum_{i\in [K]}  p_{t}(i) (\hat{\ell}_{t}(i))^2 +\frac{\log(K)}{\gamma_{T}},
\end{split}
\end{equation}
where $T=\max \{T(i^*),\tau(E)\}$.
We have%Similar to the derivation of (\ref{eq:adv2})
%\begin{equation}\label{eq:adv3}
%    \begin{split}
%      &\mathbf{E}\Bigg[\sum_{t=1}^{\tau(E)}\mathbf{E}[\sum_{i\in [K]} p_{t}(i) \hat{\ell}_{t}(i)|\mathcal{F}_{t-1}]\Bigg]-
%      \mathbf{E}\Bigg[\sum_{t=1}^{T(i^{*})}\hat{\ell}_{t}(i)\Bigg]\\
%      &\leq   \frac{\log(K)}{\gamma_{T}}+\mathbf{E}\Bigg[\sum_{t=1}^{\tau(E)+K/c_{min}}\frac{\gamma_{t}}{2}\mathbf{E}\bigg[\sum_{i\in [K]} p_{t}(i)\hat{\ell}^2_{t}(i)|\mathcal{F}_{t-1}\bigg]\Bigg],
%\end{split}
%\end{equation}
\begin{equation}\label{eq:adv3}
    \begin{split}
      &\mathbf{E}\Bigg[\sum_{t=1}^{T}\mathbf{E}[\sum_{i\in [K]} p_{t}(i) \hat{\ell}_{t}(i)|\mathcal{F}_{t-1}]\Bigg]-
      \mathbf{E}\Bigg[\sum_{t=1}^{T}\hat{\ell}_{t}(i)\Bigg]\\
      &\leq   \frac{\log(K)}{\gamma_{T}}+\mathbf{E}\Bigg[\sum_{t=1}^{T}\frac{\gamma_{t}}{2}\mathbf{E}\bigg[\sum_{i\in [K]} p_{t}(i)\hat{\ell}^2_{t}(i)|\mathcal{F}_{t-1}\bigg]\Bigg],
\end{split}
\end{equation}
where $\mathcal{F}_{t}$ is the sigma field with respect to the  entire past until round $t$.

Now, let us bound the terms in (\ref{eq:adv3}). We have
\begin{equation}\label{eq:adv4}
\begin{split}
&\mathbf{E}[\sum_{i\in [K]} p_{t}(i) \hat{\ell}_{t}(i)|\mathcal{F}_{t-1}]\\
&\geq \mathbf{E}\Bigg[\sum_{i\in[K]}(\tilde{p}_{t}(i)-\epsilon_{t}(i))\hat{\ell}_{t}(i)|\mathcal{F}_{t-1}\Bigg],\\
&\geq  \frac{1}{c_{min}}-\mathbf{E}\Bigg[\frac{r_{t}(i_t)}{c_{t}(i_t)}\Bigg|\mathcal{F}_{t-1}\Bigg]-\sum_{i\in[K]}\frac{\epsilon_{t}(i)}{c_{min}}.
\end{split}
\end{equation}
Also,
\begin{equation}\label{eq:adv44}
     \mathbf{E}\Bigg[\sum_{t=1}^{T}\hat{\ell}_{t}(i^*)\Bigg]=\sum_{t=1}^{T}\frac{1}{c_{min}}-\sum_{t=1}^{T}\frac{r_{t}(j)}{c_{t}(j)}.
\end{equation}
Additionally,
\begin{equation}\label{eq:adv5}
\begin{split}
&\mathbf{E}\Bigg[\sum_{i\in[K]}p_{t}(i)\hat{\ell}^2_{t}(i)|\mathcal{F}_{t-1}\Bigg]\\
&\leq \mathbf{E}\Bigg[\sum_{i\in[K]}\frac{p_{t}}{c_{min}^2\tilde{p}^2_{t}}\Bigg|\mathcal{F}_{t-1}\Bigg],\\
&\leq \sum_{i\in[K]}\frac{p_{t}}{c_{min}^2\tilde{p}_{t}},\\
&\stackrel{(a)}{\leq}\frac{2K}{c_{min}^2},
\end{split}
\end{equation}
where last inequality follows from the definition of $\tilde{p}_{t}(i)$, and the fact that for all $i\in [K]$ and $t$,$(1-\sum_{j\neq i}\epsilon_{t}(j))\geq 0.5$.

Using (\ref{eq:adv2}),(\ref{eq:adv3}), (\ref{eq:adv4}),(\ref{eq:adv44}) and (\ref{eq:adv5}), we have that the expected regret of the algorithm is at most
\begin{equation}
\begin{split}
        &\frac{\log(K)}{\gamma_{n^\prime}}
        +\frac{K}{c_{min}^2} \sum_{t=1}^{n^\prime}\gamma_{t}+\sum_{t=1}^{\tau(E)}\sum_{i\in[K]}\frac{\epsilon_{t}(i)}{c_{min}}\\
       &\stackrel{(a)}{\leq} \frac{\log(K)}{\gamma_{n^\prime}}
        +\frac{K}{c_{min}^2} \sum_{t=1}^{n^\prime}\gamma_{t}+\sum_{t=1}^{n^\prime}\sum_{i\in[K]}\frac{\gamma_{t}}{c^2_{min}}\\
        &\stackrel{(b)}{\leq}6\sqrt{\frac{BK\log(K)}{c_{min}^3}},
\end{split}
\end{equation}
where $n^\prime=\tau(E)+K/c_{min}$,  $(a)$ follows from the value of $\gamma$, and from the fact that $\epsilon_{t}(i)\leq 0.5 c_{min}\sqrt{\log(K)/tK}$,and $(b)$ follows from the concavity of $\sqrt{x}$.
\end{proof}

\bibliographystyle{apalike}
\bibliography{knapsackBanditsWithCorruption}  

\end{document}